\documentclass[lettersize,journal]{IEEEtran}
\usepackage{amsmath,amsfonts,amsthm,amssymb}
\usepackage{algorithmic}
\usepackage{array}
\usepackage[caption=false,font=normalsize,labelfont=sf,textfont=sf]{subfig}
\usepackage{textcomp}
\usepackage{stfloats}
\usepackage{url}
\usepackage{hyperref}
\usepackage{verbatim}
\usepackage{graphicx}
\usepackage{cite}
\usepackage{multirow}
\usepackage{makecell}
\usepackage{booktabs}
\usepackage{tabularx}
\usepackage{amssymb}
\usepackage{ulem}
\usepackage[table]{xcolor}
\usepackage[ruled,linesnumbered]{algorithm2e}
\usepackage{framed}
\usepackage{rotating}
\DeclareMathOperator*{\argmax}{arg\,max}

\newtheorem{theorem}{Theorem}

\newtheorem{definition}{Definition}

\newtheorem{fact}{Fact}

\newcommand{\bTheta}{\mathbf{\Theta}}
\graphicspath{{figs/}}

\begin{document}
\normalem

\title{Scalable Structure Learning of Bayesian Networks by Learning Algorithm Ensembles}
\author{Shengcai Liu, \IEEEmembership{Member,~IEEE}, Hui Ou-yang, Zhiyuan Wang, Cheng Chen, Qijun Cai, \\Yew-Soon Ong, \IEEEmembership{Fellow,~IEEE}, and Ke Tang, \IEEEmembership{Fellow,~IEEE}
\IEEEcompsocitemizethanks{\IEEEcompsocthanksitem S. Liu, H. Ou-yang,  Z. Wang, C. Chen, Q. Cai, and K. Tang are with the Guangdong Provincial Key Laboratory of Brain-Inspired Intelligent Computation, Department of Computer Science and Engineering, Southern University of Science and Technology, Shenzhen 518055, China.
E-mail: liusc3@sustech.edu.cn (S. Liu), and tangk3@sustech.edu.cn (K. Tang).
\IEEEcompsocthanksitem Y.-S. Ong is with the Centre for Frontier AI Research, Institute of High Performance Computing, Agency for Science, Technology and Research, Singapore 138634, and the College of Computing and Data Science, Nanyang Technological University, Singapore 639798.
E-mail: asysong@ntu.edu.sg
}
\thanks{Corresponding author: K. Tang.}
}


\markboth{Journal of \LaTeX\ Class Files,~Vol.~14, No.~8, August~2021}%
{Shell \MakeLowercase{\textit{et al.}}: A Sample Article Using IEEEtran.cls for IEEE Journals}

\maketitle

\begin{abstract}
Learning the structure of Bayesian networks (BNs) from data is challenging, especially for datasets involving a large number of variables.
The recently proposed divide-and-conquer (D\&D) strategies present a promising approach for learning large BNs.
However, they still face a main issue of unstable learning accuracy across subproblems.
In this work, we introduce the idea of employing structure learning ensemble (SLE), which combines multiple BN structure learning algorithms, to consistently achieve high learning accuracy.
We further propose an automatic approach called Auto-SLE for learning near-optimal SLEs, addressing the challenge of manually designing high-quality SLEs.
The learned SLE is then integrated into a D\&D method.
Extensive experiments firmly show the superiority of our method over D\&D methods with single BN structure learning algorithm in learning large BNs, achieving accuracy improvement usually by 30\%$\sim$225\% on datasets involving 10,000 variables.
Furthermore, our method generalizes well to datasets with many more (e.g., 30000) variables and different network characteristics than those present in the training data for learning the SLE.
These results indicate the significant potential of employing (automatic learning of) SLEs for scalable BN structure learning.
\end{abstract}


\section{Introduction}
Learning the structure of Bayesian networks (BNs)~\cite{pearl1985bayesian} from data has attracted much research interest, due to its wide applications in machine learning, statistical modeling, and causal inference~\cite{pearl1988probabilistic,JinYJPHWYZ23,kitson2023survey}.
Various methods have been proposed to tackle this problem, including constraint-based methods~\cite{colombo2014order}, score-based methods~\cite{ramsey2017million}, and hybrid methods~\cite{tsamardinos2006max}.
However, most previous studies primarily dealt with a relatively small number of variables.
For example, the \textit{bnlearn} repository~\cite{scutari2010learning}, which is widely used in the literature, contains mostly networks with only a few dozen nodes (variables).
In comparison, in real-world applications such as alarm events analysis~\cite{cai2022thps}, MRI image interpretation \cite{ramsey2017million}, and human genome analysis \cite{schaffter2011genenetweaver}, it is common to generate and collect data from thousands of variables and beyond.
Unfortunately, as the number of variables increases, many of the existing methods would scale poorly, exhibiting significant slowdowns and reduced accuracy~\cite{zhu2021efficient,WangBCLZC25}.

Recently, divide-and-conquer methods (D\&C)~\cite{gu2020learning,ZhangZYGWZH22,WangBCLZC25} have been introduced to enhance the scalability of BN structure learning, particularly for large BNs.
These approaches generally involve three steps: partitioning nodes into clusters (partition), learning a subgraph on each cluster (estimation), and subsequently fusing or merging all subgraphs into a single BN (fusion).
The reduction in the number of nodes within each cluster, compared to the total network, accelerates the overall structure learning process.
Additionally, the structure learning for different subproblems can be readily parallelized, leading to improved computational efficiency.
Furthermore, these D\&C methods are flexible, as any existing structure learning algorithm can be employed in the estimation step. 

However, despite the evident advantages provided by D\&C methods, they still face a main issue of unstable structure learning accuracy across subproblems.
The root cause for this is that the partition step of these methods may yield subproblems with significantly different characteristics~\cite{gu2020learning}, e.g., varying node numbers.
When a single structure learning algorithm is used to solve all subproblems, as in existing D\&C methods~\cite{gu2020learning,ZhangZYGWZH22,WangBCLZC25}, achieving stable learning accuracy across different subproblems becomes challenging.
In fact, even for the same algorithm, different parameter values can lead to significant variations in its behavior, thereby affecting its suitability for solving specific subproblems.
For instance, when employing the well-known fast greedy equivalence search (fGES)~\cite{ramsey2017million} algorithm with Bayesian information criterion (BIC) as the score function, its penalty coefficient should ideally increase with the sparsity of the underlying BN.
However, given a BN structure learning problem, determining the optimal penalty coefficient in advance is difficult as the underlying BN is unknown.

Inspired by the success of ensemble methods like AutoAttack~\cite{Croce020a} in the field of adversarial robustness, which utilizes an attack ensemble to achieve more reliable robustness evaluation compared to individual attacks, we introduce the idea of employing algorithm ensemble, specifically structure learning ensemble (SLE), to achieve stable learning accuracy across subproblems in BN structure learning.
Specifically, a SLE comprises several structure learning algorithms, dubbed member algorithms. 
When applied to a BN structure learning problem, a SLE runs its member algorithms individually and chooses the best of their outputs.
Similar to how AutoAttack integrates complementary attacks, a high-quality SLE should consist of complementary member algorithms that excel at solving different types of problems.
However, \textit{in contrast to AutoAttack where the attacks are manually designed and selected, we propose to automatically learn SLEs from data}. 
This can significantly reduce the reliance on human expertise and effort, as manually constructing SLEs typically requires domain experts to explore the vast design space of SLEs, which can be both laborious and intricate.

\begin{figure}[tbp]
	\centering
	\includegraphics[width=\linewidth]{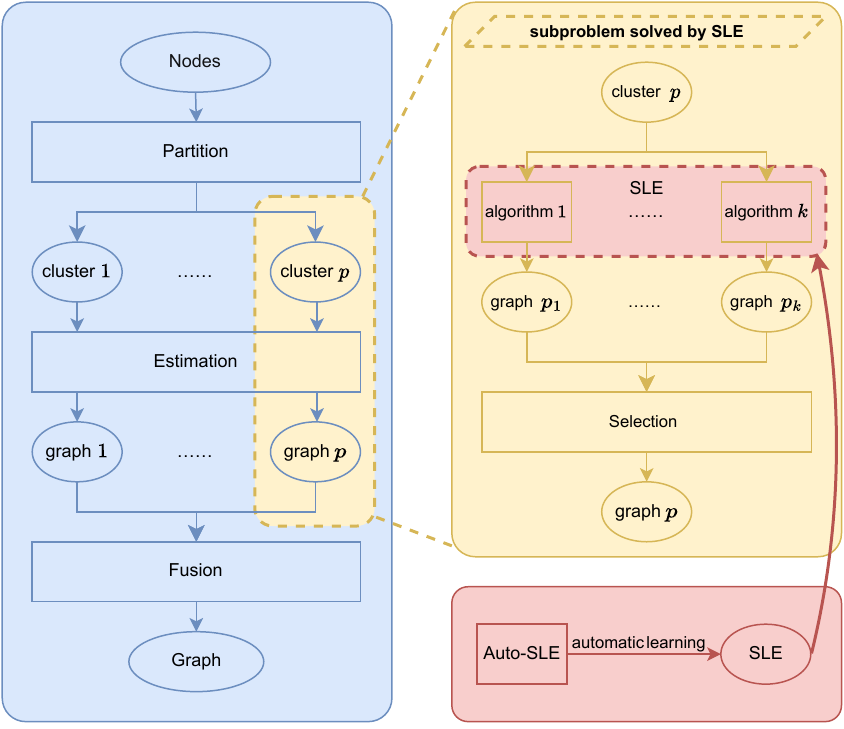}
	\caption{An overview of P/SLE. The SLE is automatically learned by Auto-SLE and integrated into the estimation step of PEF.}
	\label{fig:overview}
\end{figure}

Specifically, we first formulate the SLE learning problem, and then present Auto-SLE, a simple yet effective approach to address this problem.
This approach is implemented with a design space containing several candidate algorithms, each with a set of parameters.
To learn a SLE, Auto-SLE begins with an empty ensemble and iteratively adds the algorithm and its associated parameter values to the ensemble such that the ensemble improvement is maximized.
Beyond its conceptual simplicity, Auto-SLE is also appealing due to its ability to yield SLEs that are provably within a constant factor of optimal for the training problem set.
Finally, the SLE learned by Auto-SLE can be integrated into the estimation step of existing D\&C methods.
In this work, it is integrated into the partition-estimation-fusion (PEF) method~\cite{gu2020learning}.
The resultant method is dubbed P/SLE (see Figure~\ref{fig:overview} for an overview).
 
From extensive comparisons, it is found that P/SLE consistently achieves significantly higher accuracy in learning large BNs compared to existing PEF methods that utilize a single state-of-the-art algorithm in their estimation step. 
This superiority becomes more pronounced as the network size further increases.
On datasets involving 10,000 variables, P/SLE typically achieves accuracy improvement ranging from 30\%$\sim$225\%.
Furthermore, without any additional tuning or adaptation, P/SLE generalizes well to datasets with a much larger number (e.g., 30,000) of variables and different network characteristics than those present in the training data.

It should be noted that D\&C methods are generally most suitable for learning BNs with a block structure to some extent, meaning the connections between subgraphs are relatively weak.
It is quite common for a large network to exhibit such a block structure, due to the underlying heterogeneity among the nodes~\cite{holland1983stochastic,airoldi2008mixed,abbe2015exact}.
On the other hand, our results also indicate that, even for large BNs without a block structure, P/SLE can still achieve competitive learning accuracy.

It is worth mentioning that Auto-SLE can be applied to various types of BNs (Gaussian and non-Gaussian) and even other types of causal models.
In this work, we focus on Gaussian BNs, which are arguably the most widely studied type of BNs and have a rich set of baselines, allowing us to thoroughly assess the potential of Auto-SLE.
While PEF was chosen as the embedding D\&C method in this work, the learned SLE can also be integrated into other D\&C methods, such as the ones presented in~\cite{ZhangZYGWZH22,WangBCLZC25}.
Finally, the promising performance of P/SLE not only demonstrates the significant improvement that SLEs can bring to D\&C methods for scalable BN structure learning -- a largely unexplored area -- but also highlights the effectiveness of Auto-SLE for learning SLEs.

The main contributions of this article are summarized below.
\begin{itemize}
	\item The idea of embedding SLEs into the estimation step of D\&C methods is introduced to enhance the accuracy of large-scale BN structure learning.
    \item A novel formulation for the SLE learning problem is proposed, along with Auto-SLE, an approach for automatically learning SLEs.
    The effectiveness of this approach is further supported by theoretical analysis.
    \item The SLE learned by Auto-SLE is integrated into the PEF method, resulting in P/SLE.
    Its effectiveness and generalization capabilities are empirically validated through extensive experiments on large-scale BN structure learning problems.
\end{itemize}

The rest of this article is organized as follows.
Section~\ref{sec:pre_related_work} presents preliminaries and related work.
Section~\ref{sec:learning_SLEs} formally defines the SLE learning problem, followed by the proposed approach Auto-SLE.
Section~\ref{sec:exp} presents the experimental results and their analysis.
Finally, Section~\ref{sec:conclusion} concludes the article.


\section{Preliminaries and Related Work}
\label{sec:pre_related_work}

\subsection{The BN Structure Learning Problem}
The structure of a BN for $d$ random variables $X_1,...,X_d$ is represented by a directed acyclic graph (DAG), denoted as $\mathcal{G}=(V, E)$.
Here, $V=\{1,...,d\}$ is the set of nodes corresponding to the random variables and $E=\{(j,i) \in V \times V : j \rightarrow i\}$ is the directed edge set.
Define $V_{\mathrm{pa}(i)}=\{ j \in V:(j,i) \in E \}$ as the parent node set of node $i$, and $X_{\mathrm{pa}(i)}$ as the set of corresponding random variables.
The joint probability density function $f$ of $X_1,...,X_d$ is factorized according to the structure of $\mathcal{G}$: 
\begin{equation}
f \left( X_1, X_2,...,X_d \right) = \prod_{i=1}^{d} f \left( X_{i}|X_{\mathrm{pa}(i)} \right),
\end{equation}
where $ f\left( X_{i}|X_{\mathrm{pa}(i)} \right)$ is the conditional probability density of $X_i$ given $X_{\mathrm{pa}(i)}$.
In this work, we focus on Gaussian BNs for continuous data.
Specifically, the conditional distributions are specified by the following linear structural equation model:
\begin{equation}
\label{eq:lsem}
	X_i=\phi \left( X_{\mathrm{pa}(i)} \right) + \varepsilon_i, \ \ \ i=1,...,d,
\end{equation}
where $\phi(\cdot)$ denotes a linear function and $\varepsilon_i \sim \mathcal{N}\left(0, \sigma_j^2\right)$.
Suppose we have obtained $m$ iid observations of $X_{1}, \dots, X_{d}$, denoted as $D \in \mathbb{R}^{m \times d}$.
Given $D$, the goal is to learn a DAG structure $\mathcal{G}=(V, E)$ that accurately reflects the conditional dependencies among $X_1,\dots,X_d$.
In practice, accuracy metrics such as F1 score and structural Hamming distance (SHD) are typically used to assess the quality of the learned BN.

\subsection{Existing Methods for BN Structure Learning}
The BN structure learning problem has been proven to be NP-hard \cite{chickering2004large}, leading to main research efforts on developing approximation methods to solve it.
These methods can be broadly classified into constraint-based, score-based, and hybrid methods.
Constraint-based methods, such as PC \cite{spirtes2000causation}, MMPC \cite{tsamardinos2003time}, and PC-Stable \cite{colombo2014order}, use conditional independence tests on observations to identify relationships among variables.
In comparison, score-based methods explore the space of DAGs or Markov equivalence classes (MECs) using search heuristics such as tabu search (TS)~\cite{bouckaert1995bayesian}, genetic algorithm (GA) \cite{larranaga1996learning,LeeK20,YanFLZSW23}, and greedy search~\cite{chickering2002learning}.
These methods also employ score functions, such as BDeu \cite{akaike1974new}, BIC \cite{schwarz1978estimating}, and K2 \cite{cooper1992bayesian}, to guide the search.
It is worth mentioning a recent research line of score-based methods including NOTEARS \cite{zheng2018dags} and LEAST \cite{zhu2021efficient} that reformulate the structure learning as a continuous optimization problem.
Finally, hybrid methods integrate both constraint-based and score-based techniques.
For example, MMHC~\cite{tsamardinos2006max} uses MMPC to build the graph skeleton and utilizes TS to determine the final BN.

However, as the number of variables increases, many of the existing methods would slow down dramatically and become much less accurate~\cite{zhu2021efficient,WangBCLZC25}.
Actually, based on our preliminary testing of 15 existing methods, fGES~\cite{ramsey2017million}, which is a variant of the greedy search algorithm~\cite{chickering2002learning}, and PC-Stable~\cite{colombo2014order}, are the only methods capable of achieving F1 score of around 0.5 within reasonable runtime when the variable number reaches 1000 (see Section~\ref{sec:baseline_setting}).
Moreover, when the number of variables reaches 10,000, all methods, even after running for a quite long period of time (e.g., 24 hours), are unable to output solutions.




\subsection{D\&C Methods for Learning Large BNs}
To address the challenge of learning large BNs, a series of divide-and-conquer (D\&C) methods~\cite{gu2020learning,ZhangZYGWZH22,WangBCLZC25} have been proposed.
The strategy of these methods involves breaking down the large BN learning into smaller, more manageable sub-problems, solving these individually, and then integrating the results to form the overall network.
Since structure learning ensembles (SLEs) are embedded into the partition-estimation-fusion (PEF) method in this work, PEF is briefly introduced here.
The details of PEF can be found in the supplementary.
Introduced by Gu and Zhou~\cite{gu2020learning}, PEF comprises three steps:
\begin{itemize}
\item \textbf{Partition}: The $d$ nodes are partitioned into clusters with a hierarchical clustering algorithm.
\item \textbf{Estimation}: An existing structure learning algorithm is applied to estimate a subgraph on each cluster of nodes.
\item \textbf{Fusion}: Estimated subgraphs are merged into one DAG containing all the $d$ nodes.
\end{itemize}

While PEF and other D\&C methods have dramatically enhanced the capabilities of learning large BNs, they still face the unstable structure learning accuracy across subproblems.
In the following section, the use of automatically learned SLE in the estimation step to address this issue will be described.

\section{Automatic Learning of SLEs}
\label{sec:learning_SLEs}

\subsection{Problem Formulation}
We first formulate the SLE Learning problem.
Formally, a SLE with $k$ member algorithms is denoted as $\mathbb{A}=\{\theta_1,\dots,\theta_k\}$, where $\theta_i$ represents the $i$-th member algorithm of $\mathbb{A}$.
Let $T=\{D_1, D_2, \dots \}$ denote a training problem set, where each $D_i$ represents a BN structure learning problem with known ground truth.
When using $\mathbb{A}$ to solve a problem $D \in T$, one straightforward strategy is to run all member algorithms of $\mathbb{A}$ individually in parallel, and the best solution among all the found solutions in terms of a quality measure $Q$ (e.g., F1 score) is returned.
Let $Q(\mathbb{A}, D)$ and $Q(\mathbb{\theta}_i, D)$ denote the performance of $\mathbb{A}$ and $\theta_i$ on $D$ in terms of $Q$, respectively. 
Without loss of generality, we assume a larger value is better for $Q$.
Then we have:
\begin{equation}
	Q(\mathbb{A}, D) = \max_{\theta_i \in \mathbb{A}} Q(\mathbb{\theta}_i, D).
\end{equation}
Then the performance of $\mathbb{A}$ on $T$ in terms of $Q$, denoted as $Q(\mathbb{A}, T)$,  is the average value of $\mathbb{A}$'s performance on the training problems in $T$ (Note $Q(\mathbb{A}, T)=0$ when $\mathbb{A}$ is empty):
\begin{equation}
	Q(\mathbb{A}, T)= \frac{1}{|T|} \sum_{D \in T} Q(\mathbb{A}, D).
\end{equation}

For the automatic learning of $\mathbb{A}$, the member algorithms of $\mathbb{A}$ are not manually determined but automatically identified from an algorithm configuration space $\bTheta$.
Specifically, suppose we have a candidate algorithm pool $\{\mathcal{A}_1,...,\mathcal{A}_n\}$, which can be constructed by collecting existing BN structure learning algorithms.
Each candidate algorithm has some parameters.
Let $\Theta_i$ denote the parameter configuration space of $\mathcal{A}_i$, where a configuration $\theta \in \Theta_i$ refers to a setting of $\mathcal{A}_i$'s parameters, such that its behaviors is completely specified.
Then, the algorithm configuration space $\bTheta$ is defined as $\bTheta = \Theta_1 \cup \Theta_2 \dots \cup \Theta_n $, and each $\theta \in \bTheta$ represents a specific candidate algorithm along with the specific values of its parameters.
As presented in Definition~\ref{def:prob_formulation}, the SLE learning problem is to find $k$ member algorithms from $\bTheta$ to form a SLE $\mathbb{A}^*$, such that its performance on the training set $T$ in terms of $Q$ is maximized.
\begin{definition}
	\label{def:prob_formulation}
	Given $T, Q, \bTheta$, and $k$, the SLE learning problem is to find $\mathbb{A}^*=\{ \theta_1^{*} \dots\ \theta_k^{*} \}$ that maximizes $Q(\mathbb{A}^*, T)$, s.t. $\theta_i^{*} \in \bTheta$ for $i=1 \dots k$.
\end{definition}

\begin{algorithm}[tbp]
	\LinesNumbered
	\KwIn{quality measure $Q$, training set $T$, algorithm configuration space $\bTheta$, ensemble size $k$}
	\KwOut{$\mathbb{A}$}
	$\mathbb{A} \leftarrow \varnothing$, $i \leftarrow 1$;\\
	\While{$i \leq k$}
	{
		$\theta^{\circ} \leftarrow \argmax_{\theta \in \bTheta} Q\left(\mathbb{A} \cup \{\theta\}, T \right) - Q\left(\mathbb{A}, T\right)$;\\
		\lIf{$Q \left(\mathbb{A} \cup \{\theta^{\circ}\}, T \right) = Q\left(\mathbb{A}, T\right)$}{return $\mathbb{A}$}
		$\mathbb{A} \leftarrow \mathbb{A} \cup \{\theta^{\circ}\} $, $i \leftarrow i+1$;\\
	}
	\Return{$\mathbb{A}$}
	\caption{Auto-SLE}
	\label{alg:auto-sle}
\end{algorithm}

\subsection{Auto-SLE: A Greedy Learning Approach}
\label{sec:approach}
We now introduce Auto-SLE, a simple yet effective approach for automatically learning SLEs.
As shown in Algorithm~\ref{alg:auto-sle}, Auto-SLE starts with an empty ensemble $\mathbb{A}$ (line 1) and finds the candidate algorithm and its parameter values denoted as $\theta^{\circ}$ that, if included in $\mathbb{A}$, maximizes ensemble improvement in terms of $Q$ (line 3).
Ties are broken arbitrarily here.
After this, $\theta^{\circ}$ is subject to the following procedure:
if adding it to $\mathbb{A}$ does not improve performance, which means the learning process has converged, Auto-SLE will terminate and return $\mathbb{A}$ (line 4);
otherwise $\theta^{\circ}$ is added to $\mathbb{A}$ (line 5).
The above process will be repeated until $k$ member algorithms have been found (line 2).

Let $\Delta(\theta|\mathbb{A})=Q\left(\mathbb{A} \cup \{\theta\}, T \right) - Q\left(\mathbb{A}, T\right)$ denote the performance improvement brought by adding $\theta$ to $\mathbb{A}$.
Noticing that each iteration of Auto-SLE needs to find $\theta^{\circ}$ that maximizes $\Delta(\theta|\mathbb{A})$, when the algorithm configuration space $\bTheta$ is large or even infinite (e.g., candidate algorithms have continuous parameters), using enumeration to find $\theta^{\circ}$ is impractical.
In practice, we employ black-box parameter optimization procedures, such as Bayesian optimization~\cite{lindauer2022smac3}, to approximately maximize $\Delta(\theta|\mathbb{A})$.

\subsection{Theoretical Analysis}
We now theoretically analyze the performance of Auto-SLE on a give training problem set.
For notational simplicity, henceforth we omit the $T$ in $Q(\cdot, T)$ and directly use $Q(\cdot)$.
Our analysis is based on the following key fact that $Q(\cdot)$ is monotone and submodular.
\begin{fact}
	\label{fact:sub}
	$Q(\cdot)$ is a monotone submodular function, i.e., for any two SLEs $\mathbb{A},\mathbb{A}' \subset \bTheta $ and any $\theta \in  \bTheta$, it holds that $Q(\mathbb{A}) \leq Q(\mathbb{A} \cup \mathbb{A}')$ and $Q\left(\mathbb{A} \cup \mathbb{A}' \cup \{\theta\} \right)-Q\left(\mathbb{A} \cup \mathbb{A}'\right) \leq Q\left(\mathbb{A} \cup \{\theta \} \right)-Q\left(\mathbb{A}\right)$.
\end{fact}
\begin{proof}
	By definition, $Q(\mathbb{A}, D) = \max_{\theta \in \mathbb{A}} Q(\theta, D)$, then it holds that $Q(\mathbb{A} \cup \mathbb{A}', D) = \max_{\theta \in \mathbb{A} \cup \mathbb{A}'}Q(\theta, D) \geq \max_{\theta \in \mathbb{A}, }Q(\theta, D)$. The monotonicity holds.

	To prove submodularity, we have
	\begin{equation}
	\begin{aligned}
	\hspace{-2cm}Q\left(\mathbb{A} \cup \{\theta \} \right)-Q(\mathbb{A}) 
	&= \begin{aligned}[t] & \frac{1}{|T|} \sum_{D \in T} [Q(\mathbb{A} \cup \{\theta\}, D) - Q(\mathbb{A}, D) ]\hspace{-2cm}  \\& \quad \text{by definition of } Q(\cdot)\end{aligned}\\
	&= \begin{aligned}[t] &\frac{1}{|T|} \sum_{D \in T} [Q(\theta, D) - Q(\mathbb{A},D)]^{+}\hspace{-2cm}\end{aligned}\\
	&\geq \begin{aligned}[t] &\frac{1}{|T|} \sum_{D \in T} [Q(\theta, D) - Q(\mathbb{A} \cup \mathbb{A}',D)]^{+}\hspace{-2cm} \\& \quad \text{by monotonicity}\end{aligned}\\
	&= \begin{aligned}[t] &Q\left(\mathbb{A} \cup \mathbb{A}' \cup \{\theta\} \right)-Q\left(\mathbb{A} \cup \mathbb{A}'\right).\hspace{-2cm}\end{aligned}\\
	\end{aligned}
	\end{equation}
	The proof is complete.
\end{proof}
Intuitively, $Q$ exhibits a diminishing returns property that the marginal gain of adding $\theta$ diminishes as the ensemble size increases.
Based on Fact~\ref{fact:sub}, Theorem~\ref{theorem} holds.
\begin{theorem}
\label{theorem}
Using a parameter optimization procedure that, in each iteration of Auto-SLE, returns $\hat{\theta}$ within $\epsilon$-absolute error of the maximum of $\Delta(\theta | \mathbb{A})$, i.e., $\Delta(\hat{\theta} | \mathbb{A}) \geq \Delta(\theta^{\circ} | \mathbb{A}) - \epsilon$, then the quality $Q(\mathbb{A})$ of the SLE learned by Auto-SLE is bounded by
\begin{equation}
	Q(\mathbb{A}) \geq (1-1/e) \cdot Q(\mathbb{A}^*) - k\epsilon,
\end{equation}
where $\mathbb{A}^*$ is the optimal SLE to the SLE learning problem in Definition~\ref{def:prob_formulation}.
Alternatively, if $\hat{\theta}$ is within $\epsilon$-relative error of $\Delta(\theta^{\circ} | \mathbb{A})$, i.e.,  $\Delta(\hat{\theta} | \mathbb{A}) \geq \Delta(\theta^{\circ} | \mathbb{A}) \cdot (1-\epsilon)$,  then the quality $Q(\mathbb{A})$ of the SLE learned by Auto-SLE is bounded by
\begin{equation}
	Q(\mathbb{A}) \geq (1-1/e^{1-\epsilon}) \cdot Q(\mathbb{A}^*).
\end{equation} 
\end{theorem}
The proof is a slight extension of the classical derivations in maximizing $\Delta(\theta | \mathbb{A})$ ~\cite{NemhauserWF78}.
\begin{proof}
Order the candidate algorithms in $\mathbb{A}^*$ as $\{ \theta_1^{*} \dots\ \theta_k^{*} \}$.
We denote $\mathbb{A}=\{\theta_1, \theta_2, \dots, \theta_k\}$ where $\theta_i$ is the algorithm added to $\mathbb{A}$ in the $i$-th iteration of Auto-SLE.
Let $\mathbb{A}_i = \{\theta_1,\dots,\theta_i\}$ and let $\Delta(\theta|\mathbb{A})=Q\left(\mathbb{A} \cup \{\theta\} \right) - Q\left(\mathbb{A}\right)$ denote the performance improvement brought by adding $\theta$ to $\mathbb{A}$.

In the first case where $\Delta(\hat{\theta} | \mathbb{A}) \geq \Delta(\theta^{\circ} | \mathbb{A}) - \epsilon$, for all positive integers $i < l \leq k$, we have:
\begin{equation}
\begin{aligned}
Q(\mathbb{A}^*) 
&\leq \begin{aligned}[t]\quad&Q(\mathbb{A}^* \cup \mathbb{A}_i) \quad \text{by monotonicity}\end{aligned}\\
&= \begin{aligned}[t]\quad&Q(\mathbb{A}_i) + \sum_{j=1}^k \Delta(\theta_j^* | \mathbb{A}_i \cup \{\theta_1^*,\dots,\theta_{j-1}^* \}) \\&\quad \text{by telescoping sum}\end{aligned}\\
&\leq \begin{aligned}[t]\quad&Q(\mathbb{A}_i) + \sum_{\theta \in \mathbb{A}^*} \Delta(\theta | \mathbb{A}_i) \quad \text{by submodularity}\hspace{-2cm}\end{aligned}\\
&\leq \begin{aligned}[t]\quad&Q(\mathbb{A}_i) + \sum_{\theta \in \mathbb{A}^*} \Delta(\theta^{\circ} | \mathbb{A}_i) \quad \text{by definition of } \theta^{\circ}\end{aligned}\hspace{-2cm}\\
&\leq \begin{aligned}[t]\quad&Q(\mathbb{A}_i) + \sum_{\theta \in \mathbb{A}^*} \left( Q(\mathbb{A}_{i+1}) - Q(\mathbb{A}_i)  + \epsilon \right) \\&\quad \text{by }  \Delta(\theta^{\circ} | \mathbb{A}) \leq \Delta(\hat{\theta} | \mathbb{A}) + \epsilon \end{aligned}\\
&\leq \begin{aligned}[t]\quad&Q(\mathbb{A}_i) + k \left( Q(\mathbb{A}_{i+1}) - Q(\mathbb{A}_i)  + \epsilon \right)\end{aligned}.
\end{aligned}
\end{equation}
Let $\delta_i=Q(\mathbb{A}^*) - Q(\mathbb{A}_i)$, which allows us to rewrite the above equation as $\delta_i \leq k(\delta_i - \delta_{i+1} + \epsilon)$, then $\delta_{i+1} \leq (1 - \frac{1}{k}) \delta_i + \epsilon$.
Hence, we have
\begin{equation}
\begin{aligned}
\delta_{l} 
&\leq (1 - \frac{1}{k})^l \delta_0 + k\epsilon \cdot [1 - (1-\frac{1}{k})^l]\\
&\leq e^{-l/k} \delta_0 + k\epsilon \quad \text{by } 1-x \leq e^{-x} \text{ for all } x\in \mathbb{R}\\
&= e^{-l/k} (Q\left(\mathbb{A}^*) - Q(\mathbb{A}_0)\right) + k\epsilon\\
&= e^{-l/k} Q(\mathbb{A}^*) + k\epsilon \quad \text{ by that } \mathbb{A}_0=\varnothing \text{ and } Q(\mathbb{A}_0)=0\hspace{-2cm}.
\end{aligned}
\end{equation}
Rearranging $\delta_l=Q(\mathbb{A}^*) - Q(\mathbb{A}_l) \leq e^{-l/k} Q(\mathbb{A}^*) + k\epsilon$, we have
\begin{equation}
	Q(\mathbb{A}_l) \geq (1-e^{-l/k}) \cdot Q(\mathbb{A}^*) - k\epsilon.
\end{equation}
Since the SLE learned by Auto-SLE is $\mathbb{A}_k$, then we have
\begin{equation}
Q(\mathbb{A}_k) \geq (1-1/e) \cdot Q(\mathbb{A}^*) - k\epsilon.
\end{equation}

In the second case where $\Delta(\hat{\theta} | \mathbb{A}) \geq \Delta(\theta^{\circ} | \mathbb{A}) \cdot (1-\epsilon)$,
similarly, for all positive integers $i < l \leq k$, we have:
\begin{equation}
\begin{aligned}
	Q(\mathbb{A}^*) 
	&\leq \begin{aligned}[t]\quad&Q(\mathbb{A}_i) + \sum_{\theta \in \mathbb{A}^*} \left( Q(\mathbb{A}_{i+1}) - Q(\mathbb{A}_i)  \right) / (1-\epsilon)\hspace{-2cm} \\\quad &\text{by }  \Delta(\theta^{\circ} | \mathbb{A}) \leq \Delta(\hat{\theta} | \mathbb{A}) / (1-\epsilon)\end{aligned}\\
	&= \begin{aligned}[t]\quad&Q(\mathbb{A}_i) + \frac{k}{1-\epsilon} \left( Q(\mathbb{A}_{i+1}) - Q(\mathbb{A}_i)  \right)\end{aligned}.
\end{aligned}
\end{equation}
Similarly, let $\delta_i=Q(\mathbb{A}^*) - Q(\mathbb{A}_i)$ and use the above procedure, we have
\begin{equation}
Q(\mathbb{A}_l) \geq (1-e^{-l(1-\epsilon)/k}) \cdot Q(\mathbb{A}^*).
\end{equation}
Let $l=k$, we have
\begin{equation}
Q(\mathbb{A}_k) \geq (1-1/e^{1-\epsilon}) \cdot Q(\mathbb{A}^*).
\end{equation}
The proof is complete.
\end{proof}

Based on Theorem~\ref{theorem}, Auto-SLE achieves $(1-1/e)$-approximation for the optimal quality when given a perfect parameter optimization procedure with $\epsilon=0$.
Suboptimal parameter optimization procedures result in worse outcomes but small errors $\epsilon$ do not escalate.
This is important because with large algorithm configuration space, it cannot be expected for blackbox optimization procedures to find $\theta^\circ$ in realistic time.
However, at least for some scenarios with a few parameters, widely used Bayesian optimization techniques such as SMAC~\cite{lindauer2022smac3} have empirically been shown to yield performance close to optimal within reasonable time~\cite{LiuT019,LiuZTY23,tang2024learn}.

Note that Theorem~\ref{theorem} only provides a performance guarantee for the training problem set used for learning the SLE.
By employing the same techniques introduced by~\cite{LiuTL020}, we can recover guarantees for an independent testing set that is sampled or generated from the same distribution as the training set, given a sufficiently large training set.
Nevertheless, our experimental results (see Section~\ref{sec:results}) demonstrate that the learned SLE performs well beyond the training set.

\subsection{Using the Learned SLE to Solve a New Problem}
Importantly, when presented with a testing problem to solve, it cannot be assumed that the ground truth is known.
Thus, quality measures that do not require ground truth, such as the BIC score adopted in our experiments, are used to select the best output from the outputs of member algorithms.
It is worth mentioning that although this work focuses on utilizing the learned SLE to enhance D\&C methods, the SLE itself can also serve 
as a complete and independent BN structure learning method.

\section{Experiments}
\label{sec:exp}
Extensive experiments are conducted on both synthetic and real-world datasets to investigate the effectiveness of the proposed approach.
The experiments mainly aim to answer two key research questions (RQs):
\begin{enumerate}
\item \textbf{RQ1}: Does the SLE learned by Auto-SLE substantially enhance D\&C method, i.e., PEF, in learning large BNs?
\item \textbf{RQ2}: Can the learned SLE generalize to larger problem sizes and different network characteristics than that present in the training set?
\end{enumerate}
All the experiments are conducted on a Linux server equipped with an Intel(R)
Xeon(R) Gold 6336Y CPU @ 2.40GHz, 96 cores, and 768GB of memory.
All the data and codes are open sourced at \url{https://github.com/huiouyang16/Auto-SLE}.

\subsection{Experimental Setup}
\label{sec:exp_setup}
\subsubsection{Datasets}
Since existing public datasets typically involve limited numbers (usually much smaller than 1000) of variables, 
we generate diverse and large-scale datasets following the approach in~\cite{gu2020learning}.
Specifically, the approach consists of four steps:
(i) select an existing network structure and replicate it until a predefined variable number is reached; 
(ii) connect the replicas by adding $10\%$ of edges between them randomly, while ensuring the final network remains a DAG; 
(iii) utilize the complete DAG and Eq.~(\ref{eq:lsem}) to generate observations, where the weights in the linear function and the standard deviations of Gaussian noises are sampled uniformly from $\left[-1,-0.5\right] \cup \left[0.5,1\right]$ and $\left[0,1\right]$, respectively;
(iv) re-scale the observations such that all data columns have the same mean and standard deviation.

We select 10 networks from the \textit{bnlearn} repository~\cite{scutari2010learning} with node numbers ranging from 5 to 441.
Based on each of them, we use the above approach to generate testing problems with around 1000 and 10,000 variables.
Following~\cite{gu2020learning}, the sample size $m$ in each testing problem is set to 1000.

\subsubsection{Evaluation Metrics}
The widely-used F1 score and SHD are adopted as the metrics for assessing learning accuracy.
In line with previous comparative study~\cite{ramsey2017million}, we use two specific variants of F1 score, i.e., F1 Arrowhead ($\text{F1}^{\rightarrow}$) that considers direction and F1 Adjacent ($\text{F1}^{-}$) that ignores direction, since some methods (PC-Stable and fGES) output MECs of DAGs which may contain edges without directions. 
Specifically, considering M1 as the true MEC of the DAG and M2 as the estimated MEC, the $\text{F1}^{-}$ score is computed as the harmonic mean of precision and recall, i.e.,  2TP/(2TP+FP+FN), where TP is the count of adjacencies shared by M1 and M2, FP is the number of adjacencies in M2 but absent in M1, and FN is the number of adjacencies in M1 but not in M2.
The $\text{F1}^{\rightarrow}$ score is calculated analogously.
An arrowhead from variable A to B is considered shared between M1 and M2 if A → B exists in both.
Conversely, an arrowhead is condered present in one MEC but not he other if, for variables A and B, A → B is in one MEC while A ← B, A–B (undirected), or no adjacency between A and B exists in the other.
The SHD is defined as the number of edge insertions, deletions or flips in order to transform the learned DAG to the ground truth.
Moreover, the wall-clock runtime of the methods is reported.
For F1 metrics, a higher value is better; for SHD and runtime, a lower value is better.

\begin{table}[tbp]
    \centering
    \caption{Preliminary testing results. ``-'' means not returning solutions within a time budget of 3600s.}
      {
      \begin{tabular}{ccc}
      \toprule
      Method & $\text{F1}^{-}$ & runtime (s) \\
      \midrule
      PC-Stable & 0.71  & 576.21 \\
      \midrule
      fGES  & 0.66  & 625.86 \\
      \midrule
      HC    & 0.528 & 3170.22 \\
      \midrule
      TABU  & 0.528 & 3181.777 \\
      \midrule
      CCDr  & 0.392 & 834.781 \\
      \midrule
      MMHC  & 0.331 & 1210.805 \\
      \midrule
      RSMAX2 & 0.326 & 1416.733 \\
      \midrule
      GOLEM & 0.32  & 3622.433 \\
      \midrule
      IAMB-FDR & 0.271 & 2021.871 \\
      \midrule
      Inter-IAMB & 0.189 & 2865.472 \\
      \midrule
      IAMB  & 0.157 & 2943.536 \\
      \midrule
      Fast-IAMB & 0.107 & 3351.702 \\
      \midrule
      NOTEARS & 0.099 & 3604.744 \\
      \midrule
      GS    & -     & 3600.32 \\
      \midrule
      H2PC  & -     & 3610.257 \\
      \bottomrule
      \end{tabular}}
  	\label{tab:preliminary_testing}
  \end{table}  
\begin{table}[ht]
	\centering
		\caption{The SLE learned by Auto-SLE.}
	{
		\begin{tabular}{cccc@{}}
			\toprule
			\makecell{Member\\Algorithm} & \makecell{Candidate\\Algorithm} &\makecell{ Parameter\\Values}                                                  \\ \midrule
			1  & fGES & \makecell[lt]{$\lambda = 5.87273, m = 185$,\\without faithfulness assumption\vspace{0.15cm}} \\
			2  & fGES & \makecell[lt]{$\lambda = 20.6045, m = 17$,\\without faithfulness assumption\vspace{0.15cm}} \\
			3  & fGES & \makecell[lt]{$\lambda = 2.53767, m = 91$,\\without faithfulness assumption\vspace{0.15cm}}  \\
			4  & fGES & \makecell[lt]{$\lambda = 5.62460, m = 11$,\\without faithfulness assumption\vspace{0.15cm}}  \\ \bottomrule
		\end{tabular}
	}
	\label{table:greedyae}
\end{table}
\begin{table}[ht]
    \centering
      \caption{Default SLE.}
  {
        \begin{tabular}{ccll@{}}
            \toprule
            \makecell{Member\\Algorithm} & \makecell{Candidate\\Algorithm} & \makecell{Parameter\\Values}                                                  \\ \midrule
            1  & PC-Stable & \makecell[lt]{$\alpha = 0.05$, $m = 1000$\vspace{0.15cm}}                                        \\
            2  & fGES      & \makecell[lt]{$\lambda = 1.0$, $m = 1000$\\without faithfulness assumption\vspace{0.15cm}}             \\
            3  & PC-Stable & \makecell[lt]{$\alpha = 0.08399$, $m = 850$\vspace{0.15cm}}                          \\
            4  & fGES      & \makecell[lt]{$\lambda = 797.255$, $m = 871$\\without faithfulness assumption\vspace{0.15cm}} \\ \bottomrule
        \end{tabular}
   }
   	    \label{table:defaultae}
\end{table}
\begin{table}[ht]
    \centering
        \caption{Random SLE.}
   {
        \begin{tabular}{ccll@{}}
            \toprule
            \makecell{Member\\Algorithm} & \makecell{Candidate\\Algorithm} & \makecell{Parameter\\Values}                                                   \\ \midrule
            1  & PC-Stable & \makecell[lt]{$\alpha = 0.08399$, $m = 850$\vspace{0.15cm}}                         \\
            2  & fGES      & \makecell[lt]{$\lambda = 797.255$, $m = 871$\\without faithfulness assumption\vspace{0.15cm}} \\
            3  & PC-Stable & \makecell[lt]{$\alpha = 0.10745$, $m = 980$\vspace{0.15cm}}                          \\
            4  & fGES      & \makecell[lt]{$\lambda = 792.835$, $m = 456$\\ with faithfulness assumption\vspace{0.15cm}}   \\ \bottomrule
        \end{tabular}
   }
    \label{table:randomae}
\end{table}

\subsubsection{Learning SLE with Auto-SLE}
A diverse training problem set is beneficial for learning SLEs with good performance across problem sizes and network characteristics.
Specifically, we randomly select a base network from the 32 networks in \textit{bnlearn} and uses the afore-mentioned generation approach to generate training problems.
The training set $T$ comprises 100 problems, with variable numbers ranging from 5 to 1000.
These problems are exclusively used for learning the SLE and are independent of the testing problems. 

For the algorithm configuration space $\bTheta$, we consider two algorithms PC-stable (with two parameters) and fGES (with three parameters) as candidate algorithms because they outperform other existing methods in our preliminary testing (see Section~\ref{sec:baseline_setting}).
Details of these methods are described below.
\begin{itemize}
\item PC-Stable~\cite{colombo2014order} with two parameters: significance threshold of CI tests within the interval $\alpha \in \left[0.01, 0.2\right]$ and the search's maximum depth interval $m \in \left[1, 1000\right]$.
\item fGES~\cite{ramsey2017million} with three parameters: structural penalty of the BIC score within interval $\lambda \in \left[1.0, 1000.0\right]$, the maximum number of parents for a single node during the search process within interval $m \in \left[1, 1000\right]$, and the option to use the faithfulness assumption or not. 
\end{itemize}

We use the implementations of PC-Stable and fGES from the causal discovery tool box \textit{TETRAD}~\cite{ramsey2018tetrad}.
The sum of $\text{F1}^{-}$ and $\text{F1}^{\rightarrow}$ is adopted as the quality measure $Q$, and ensemble size $k$ is set to $4$ as running more iterations of Auto-SLE brings a negligible improvement (smaller than 0.1)  to the SLE's performance on the training set.
The Bayesian optimization tool SMAC (version 3)~\cite{lindauer2022smac3} is used to maximize $\Delta(\theta|\mathbb{A})$ in each iteration of Auto-SLE, with a time budget 12 hours per run.
Consequently, Auto-SLE consumes approximately 48 hours in total to learn the SLE, which is detailed in Table~\ref{table:greedyae}.
It is interesting to find that the SLE only contains fGES, meaning PC-Stable has not defeated fGES in the learning process.
Then, the SLE is integrated into PEF, resulting in P/SLE, which is evaluated in subsequent experiments without further tuning or adaptation.


\subsubsection{Preliminary Testing and Selection of Baselines}
\label{sec:baseline_setting}
Given the rich literature on BN structure learning~\cite{kitson2023survey}, a preliminary testing of existing methods is crucial to assess their scalability and identify state-of-the-art methods in learning large BNs, thereby enabling the selection of appropriate baselines for our experiments.
For this purpose, we collect 15 existing methods, including score-based, constraint-based, and hybrid methods, and conduct a testing of them.
These methods are listed below.
\begin{itemize}
	\item Score-based combinatorial search methods: HC~\cite{chickering2004large}, TABU~\cite{bouckaert1995bayesian}, CCDr~\cite{aragam2015concave}, fGES~\cite{ramsey2017million}
	\item Score-based continuous optimization methods: NOTEARS~\cite{zheng2018dags}, GOLEM~\cite{zhu2021efficient}
	\item Constraint-based methods: PC-Stable~\cite{colombo2014order}, GS~\cite{margaritis2003learning}, IAMB~\cite{tsamardinos2003algorithms}, Fast-IAMB~\cite{tsamardinos2003algorithms}, IAMB-FDR~\cite{pena2008learning}, Inter-IAMB~\cite{yaramakala2005speculative}
	\item Hybrid methods: MMHC~\cite{tsamardinos2006max}, RSMAX2~\cite{friedman2013learning}, H2PC~\cite{gasse2014hybrid}
\end{itemize}
Most of the implementations are collected from the \textit{bnlearn}~\cite{scutari2010learning} repository~\footnote{\url{https://www.bnlearn.com/bnrepository}};
CCDr~\footnote{\url{https://github.com/itsrainingdata/ccdrAlgorithm}}, NOTEARS~\footnote{\url{https://github.com/xunzheng/notears}}, GOLEM~\footnote{\url{https://github.com/ignavierng/golem}} are collected from Github;
fGES and PC-Stable are collected from \textit{TETRAD}~\cite{ramsey2018tetrad}~\footnote{\url{https://www.ccd.pitt.edu/tools}}.

For preliminary testing, we generate testing problems based on two random graph models, Erdös-Rényi (ER)~\cite{erdHos1960evolution} and scale-free (SF)~\cite{barabasi2003scale}, where the edge number is set to be two times the node number.
Specifically, each testing problem involves 1000 variables, has Gaussian noise, and the observation number $m=1000$.
Based on each graph model, 10 different testing problems are generated, resulting in a total of 20 testing problems.
Each method is then applied to these testing problems, with a runtime limit of 3600 seconds for each problem.
Table~\ref{tab:preliminary_testing} presents the average performance of these methods on all 20 testing problems, measured by $\text{F1}^{-}$ and runtime.
The results indicate that fGES and PC-Stable are the top-performing methods, consistently maintaining F1 scores above 0.5 when the variable number reaches 1000.
While HC and TABU also achieve F1 scores above 0.5, their runtime is significantly longer, making them impractical for testing problems involving 10,000 variables.

\begin{table*}[tbp]
	\centering
	\caption{Results on testing problems with 1000 variables, in terms of F1 Adjacent ($\text{F1}^{-}$), F1 Arrowhead ($\text{F1}^{\rightarrow}$), SHD, and runtime (T).
		On each network, the mean ± std performance obtained by each method on 10 problems is reported.
		The best performance in terms of accuracy metrics is marked with an \underline{underline}, and the performance that is not significantly different from the best performance (according to a Wilcoxon signed-rank test with significance level $p=0.05$) is indicated in \textbf{bold}.
		Let B be the best performance achieved among the baselines and A be the performance of P/SLE.
		The improvement (Impro.) ratio is calculated as (A-B)/B for F1 metrics (a higher value is better), and is calculated as (B-A)/B for SHD (a lower value is better).}
	\resizebox{1.0\textwidth}{!}{
		\begin{tabular}{@{}llcccccccccc@{}}
			\toprule
			\multicolumn{1}{c}{Problem}               &       & Alarm                     & Asia                      & Cancer                   & Child                     & Earthquake               & Hailfinder                & Healthcare                & Mildew                     & Pigs                       & Survey                    \\
			\multicolumn{1}{c}{($|V|, |E|$)}           &  & (1036,1417)               & (1000,1100)               & (1000,880)               & (1000,1375)               & (1000,880)               & (1008,1307)               & (1001,1416)               & (1015,1468)                & (1323,1954)                & (1002,1103)               \\ \midrule
			\multirow{4}{*}{P/SLE}           & $\text{F1}^{-}$  & \underline{\textbf{0.81±0.02}}  & \underline{\textbf{0.96±0.00}}  & \underline{\textbf{0.98±0.00}} & \textbf{0.86±0.01}        & \underline{\textbf{0.98±0.00}} & \underline{\textbf{0.80±0.02}}  & \textbf{0.89±0.01}                 & \underline{\textbf{0.68±0.02}}   & \underline{\textbf{0.70±0.02}}            & \underline{\textbf{0.92±0.01}}  \\
			& $\text{F1}^{\rightarrow}$ & \underline{\textbf{0.68±0.04}}  & \underline{\textbf{0.74±0.01}}  & \underline{\textbf{0.94±0.01}} & \underline{\textbf{0.60±0.02}}  & \underline{\textbf{0.94±0.01}} & \underline{\textbf{0.61±0.03}}  & 0.59±0.01           & \underline{\textbf{0.54±0.03}}   & \underline{\textbf{0.61±0.02}}   & \underline{\textbf{0.79±0.02}}  \\
			& SHD          & \underline{\textbf{625.7±73.6}} & \underline{\textbf{125.2±15.2}} & \underline{\textbf{43.1±6.9}}  & \underline{\textbf{581.1±23.3}} & \underline{\textbf{46.1±8.1}}  & \underline{\textbf{571.3±62.7}} & \underline{\textbf{485.0±38.4}} & \underline{\textbf{1135.4±71.8}} & \underline{\textbf{1262.8±84.4}} & \underline{\textbf{309.9±29.8}} \\
			& T (s)      & 7.4±0.2                   & 5.7±0.1                   & 5.1±0.1                  & 8.4±0.3                   & 5.1±0.1                  & 243.3±249.0               & 6.4±0.2                   & 8.7±0.6                    & 726.9±1176.6               & 6.0±0.1                   \\ \midrule
			\multirow{4}{*}{P/SLE(D)} & $\text{F1}^{-}$  & 0.68±0.01                 & 0.78±0.01                 & 0.74±0.01                & 0.75±0.01                 & 0.74±0.01                & 0.66±0.01                 & 0.77±0.01                 & 0.56±0.02                  & 0.53±0.01                  & 0.75±0.01                 \\
			& $\text{F1}^{\rightarrow}$ & 0.57±0.03                 & 0.64±0.01                 & 0.73±0.01                & 0.56±0.01                 & 0.72±0.01                & 0.57±0.02                 & 0.56±0.01                 & 0.45±0.03                  & 0.47±0.02                  & 0.62±0.02                 \\
			& SHD          & 1234.3±77.9               & 710.5±24.7                & 611.1±31.2               & 1055.5±39.4               & 617.5±27.1               & 1158.4±65.2               & 1051.6±31.0               & 1826.2±100.4               & 2442.4±115.6               & 871.6±32.0                \\
			& T (s)      & 26.0±1.8                  & 20.4±3.0                  & 17.1±2.7                 & 21.6±5.0                  & 15.4±2.3                 & 366.0±287.4               & 9.7±1.8                   & 14.2±2.2                   & 5697.5±1870.6              & 7.8±0.1                   \\ \midrule
			\multirow{4}{*}{P/SLE(R)}  & $\text{F1}^{-}$  & 0.72±0.02                 & 0.87±0.00                 & 0.77±0.01                & 0.83±0.01                 & 0.77±0.01                & 0.66±0.01                 & 0.87±0.01                 & 0.57±0.02                  & 0.58±0.04                  & 0.84±0.01                 \\
			& $\text{F1}^{\rightarrow}$ & 0.42±0.02                 & 0.49±0.01                 & 0.32±0.03                & 0.51±0.02                 & 0.30±0.02                & 0.44±0.02                 & 0.44±0.01                 & 0.31±0.01                  & 0.35±0.03                  & 0.46±0.01                 \\
			& SHD          & 1070.1±56.9               & 511.7±17.2                & 697.1±44.0               & 680.7±34.7                & 703.7±34.5               & 1152.3±38.7               & 750.8±31.6                & 1647.1±60.2                & 1934.3±60.6                & 675.9±28.1                \\
			& T (s)      & 19.6±1.1                  & 12.2±1.9                  & 11.4±1.9                 & 13.7±1.7                  & 10.4±1.4                 & 33.3±8.3                  & 12.9±3.1                  & 15.5±1.8                   & 5886.7±1671.0              & 8.3±2.0                   \\ \midrule
			\multirow{4}{*}{P/PC-Stable}    & $\text{F1}^{-}$  & 0.76±0.02                 & 0.91±0.00                 & 0.85±0.01                & 0.85±0.01                 & 0.85±0.01                & 0.71±0.01                 & \underline{\textbf{0.89±0.01}}                 & 0.62±0.02                  & 0.62±0.05                  & 0.89±0.01                 \\
			& $\text{F1}^{\rightarrow}$ & 0.48±0.02                 & 0.56±0.01                 & 0.47±0.03                & 0.52±0.02                 & 0.43±0.02                & 0.48±0.02                 & 0.46±0.01                 & 0.35±0.02                  & 0.40±0.04                  & 0.58±0.03                 \\
			& SHD          & 883.5±42.2                & 342.7±9.5                 & 414.8±23.8               & 585.3±29.1                & 416.8±16.4               & 922.3±41.1                & 648.0±20.3                & 1382.5±42.3                & 1627.4±84.8                & 458.0±26.2                \\
			& T (s)      & 5.5±0.8                   & 6.1±0.1                   & 6.0±0.1                  & 6.2±0.9                   & 5.7±0.4                  & 19.1±6.3                  & 2.8±0.2                   & 6.6±0.6                    & 4806.5±2311.3              & 5.3±0.8                   \\ \midrule
			\multirow{4}{*}{P/fGES}          & $\text{F1}^{-}$  & 0.68±0.01                 & 0.78±0.01                 & 0.74±0.01                & 0.75±0.01                 & 0.74±0.01                & 0.66±0.01                 & 0.77±0.01                 & 0.56±0.02                  & 0.53±0.01                  & 0.75±0.01                 \\
			& $\text{F1}^{\rightarrow}$ & 0.57±0.03                 & 0.64±0.01                 & 0.73±0.01                & 0.56±0.01                 & 0.72±0.01                & 0.57±0.02                 & 0.56±0.01                 & 0.45±0.03                  & 0.47±0.02                  & 0.62±0.02                 \\
			& SHD          & 1234.3±77.9               & 710.5±24.7                & 611.1±31.2               & 1055.5±39.4               & 617.5±27.1               & 1158.4±65.2               & 1051.6±31.0               & 1826.2±100.4               & 2442.4±115.6               & 871.6±32.0                \\
			& T (s)      & 9.3±1.1                   & 9.0±0.4                   & 8.1±0.4                  & 9.0±1.3                   & 7.9±0.4                  & 342.6±283.8               & 4.4±0.2                   & 11.7±1.7                   & 570.3±1084.5               & 6.9±0.6                   \\ \midrule
			\multirow{4}{*}{PC-Stable}        & $\text{F1}^{-}$  & 0.72±0.01                 & 0.73±0.01                 & 0.54±0.01                & 0.82±0.01                 & 0.54±0.01                & 0.67±0.01                 & 0.80±0.01                 & 0.61±0.01                  & 0.45±0.36                  & 0.68±0.01                 \\
			& $\text{F1}^{\rightarrow}$ & 0.50±0.01                 & 0.37±0.01                 & 0.18±0.00                & 0.53±0.02                 & 0.25±0.00                & 0.38±0.01                 & \underline{\textbf{0.61±0.01}}        & 0.23±0.01                  & 0.29±0.23                  & 0.44±0.01                 \\
			& SHD          & 1256.9±26.9               & 1189.2±37.0               & 1771.5±42.8              & 792.2±36.3                & 1700.1±49.5              & 1520.4±42.6               & 724.8±24.7                & 2144.2±47.3                & 1718.2±41.7                & 1165.8±40.2               \\
			& T (s)      & 383.4±65.6                & 167.7±14.5                & 176.4±19.8               & 309.2±57.3                & 165.9±12.4               & 453.8±103.8               & 172.6±14.6                & 569.4±150.5                & 43121.0±35902.9            & 169.7±15.3                \\ \midrule
			\multirow{4}{*}{fGES}              & $\text{F1}^{-}$  & 0.57±0.01                 & 0.51±0.00                 & 0.43±0.00                & 0.55±0.00                 & 0.43±0.00                & 0.47±0.01                 & 0.58±0.00                 & 0.57±0.01                  & 0.53±0.00                  & 0.50±0.00                 \\
			& $\text{F1}^{\rightarrow}$ & 0.49±0.02                 & 0.42±0.01                 & 0.42±0.00                & 0.42±0.01                 & 0.42±0.01                & 0.42±0.01                 & 0.43±0.01                 & 0.48±0.02                  & 0.48±0.01                  & 0.40±0.01                 \\
			& SHD          & 2226.1±48.7               & 2301.3±17.7               & 2320.3±16.1              & 2337.9±29.4               & 2326.9±20.0              & 2512.1±37.0               & 2331.5±25.1               & 2259.6±49.2                & 3143.3±51.1                & 2396.8±22.3               \\
			& T (s)      & 769.1±104.8               & 823.2±26.6                & 787.0±19.9               & 907.6±35.2                & 813.5±28.0               & 689.4±44.5                & 832.4±27.6                & 646.5±28.6                 & 1257.0±74.2                & 770.2±22.8                \\ \midrule
			\multirow{3}{*}{Impro. ratio}       & $\text{F1}^{-}$  & 7.4\%                     & 4.7\%                     & 14.6\%                   & 0.2\%                     & 14.6\%                   & 13.2\%                    & -0.1\%                    & 10.4\%                     & 12.5\%                     & 3.2\%                     \\
			& $\text{F1}^{\rightarrow}$ & 19.0\%                    & 16.3\%                    & 29.5\%                   & 6.8\%                     & 30.0\%                   & 6.7\%                     & -3.3\%                    & 12.3\%                     & 28.2\%                     & 26.0\%                    \\
			& SHD          & 29.2\%                    & 63.5\%                    & 89.6\%                   & 0.7\%                     & 88.9\%                   & 38.1\%                    & 25.2\%                    & 17.9\%                     & 22.4\%                     & 32.3\%                    \\ \bottomrule
		\end{tabular}
	}
	\label{tab:results_1000}
\end{table*}

Therefore, we choose fGES and PC-Stable as baselines in the comparison experiments.
Besides, fGES and PC-Stable are also integrated into the estimation step of PEF, resulting in two new baselines: P/fGES and P/PC-Stable.
Furthermore, to validate the effectiveness of Auto-SLE, we consider two alternative SLE learning approaches:
(i) default SLE, which contains the default fGES and PC-Stable, as well as variants with randomly chosen parameter values for each of them;
(ii) random SLE, which contains two variants with randomly chosen parameter values for each of fGES and PC-Stable.
Both of these SLEs consist of four member algorithms (same as our learned SLE) and are integrated into PEF, yielding two baselines P/SLE(D) and P/SLE(R).
According to the causal discovery toolbox \textit{TETRAD}~\cite{ramsey2018tetrad}, the default parameter values for PC-Stable is: $\alpha =0.05$ and $m=1000$; the default parameter values for fGES is $\lambda=1.0$, $m=1000$, without faithfulness assumption.
Details of the default SLE and random SLE are shown in Table~\ref{table:defaultae} and Table~\ref{table:randomae}.

To prevent the compared methods from running prohibitively long, a runtime limit of 24 hours is set on each testing problem.

\begin{table*}[htbp]
	\centering
	\caption{Results on problems with 10,000 variables.
		``-'' means a solution is not found within the budget of 24 hours.
		The best performance in terms of accuracy metrics is indicated in \textbf{bold}.}
		\resizebox{1.0\textwidth}{!}{
		\begin{tabular}{@{}lccccccccccc@{}}
			\toprule
			\multicolumn{1}{c}{Problem}             &         & Alarm               & Asia                & Cancer              & Child               & Earthquake          & Hailfinder          & Healthcare          & Mildew               & Pigs                & Survey              \\
			\multicolumn{1}{c}{($|V|, |E|$)}      &       & (10027, 13713)      & (10000, 11000)      & (10000, 8800)       & (10000, 13750)      & (10000, 8800)       & (10024, 12996)      & (10003, 14148)      & (10010, 14472)       & (10143, 14978)      & (10002, 11003)      \\ \midrule
			\multirow{4}{*}{P/SLE}           & $\text{F1}^{-}$  & {\textbf{0.80}} & {\textbf{0.94}} & {\textbf{0.96}} & {\textbf{0.85}} & {\textbf{0.96}} & {\textbf{0.81}} & {\textbf{0.88}} & {\textbf{0.69}}  & {\textbf{0.76}} & {\textbf{0.90}} \\
			& $\text{F1}^{\rightarrow}$ & {\textbf{0.66}} & {\textbf{0.72}} & {\textbf{0.92}} & {\textbf{0.61}} & {\textbf{0.92}} & {\textbf{0.66}} & {\textbf{0.59}} & {\textbf{0.55}}  & {\textbf{0.69}} & {\textbf{0.77}} \\
			& SHD          & {\textbf{6366}} & {\textbf{1717}} & {\textbf{705}}  & {\textbf{6017}} & {\textbf{683}}  & {\textbf{5400}} & {\textbf{5361}} & {\textbf{10911}} & {\textbf{7468}} & {\textbf{3334}} \\
			& T (s)      & 95.5                & 35.6                & 24.4                & 80.8                & 24.2                & 828.6               & 50.0                & 360.3                & 2120.5              & 39.2                \\ \midrule
			\multirow{4}{*}{P/SLE(D)} & $\text{F1}^{-}$  & 0.34                & 0.37                & 0.29                & 0.42                & 0.31                & 0.34                & 0.43                & 0.30                 & 0.33                & 0.35                \\
			& $\text{F1}^{\rightarrow}$ & 0.28                & 0.30                & 0.28                & 0.32                & 0.30                & 0.30                & 0.31                & 0.23                 & 0.29                & 0.28                \\
			& SHD          & 42805               & 36831               & 40464               & 35660               & 37413               & 40151               & 37297               & 50139                & 41958               & 39476               \\
			& T (s)      & 4084.4              & 2754.9              & 3308.6              & 3383.0              & 2810.6              & 3865.9              & 3739.6              & 4532.6               & 7040.1              & 2931.4              \\ \midrule
			\multirow{4}{*}{P/SLE(R)}  & $\text{F1}^{-}$  & 0.49                & 0.59                & 0.42                & 0.69                & 0.43                & 0.45                & 0.71                & 0.38                 & 0.02                & 0.55                \\
			& $\text{F1}^{\rightarrow}$ & 0.25                & 0.25                & 0.15                & 0.41                & 0.14                & 0.28                & 0.28                & 0.19                 & 0.01                & 0.19                \\
			& SHD          & 23673               & 18141               & 26652               & 12844               & 26065               & 25024               & 15483               & 32409                & 15551               & 21367               \\
			& T (s)      & 676.8               & 386.7               & 667.9               & 343.2               & 538.6               & 628.6               & 319.6               & 5563.9               & 7048.2              & 379.0               \\ \midrule
			\multirow{4}{*}{P/PC-Stable}    & $\text{F1}^{-}$  & 0.59                & 0.71                & 0.54                & 0.77                & 0.55                & 0.55                & 0.80                & 0.44                 & 0.12                & 0.66                \\
			& $\text{F1}^{\rightarrow}$ & 0.31                & 0.32                & 0.20                & 0.46                & 0.17                & 0.34                & 0.34                & 0.23                 & 0.06                & 0.26                \\
			& SHD          & 16842               & 11262               & 17280               & 8932                & 16869               & 17125               & 10815               & 24847                & 16005               & 14046               \\
			& T (s)      & 210.1               & 145.0               & 147.3               & 145.2               & 133.3               & 219.0               & 156.8               & 508.7                & 7005.0              & 125.4               \\ \midrule
			\multirow{4}{*}{P/fGES}          & $\text{F1}^{-}$  & 0.34                & 0.37                & 0.29                & 0.42                & 0.31                & 0.34                & 0.43                & 0.30                 & 0.33                & 0.35                \\
			& $\text{F1}^{\rightarrow}$ & 0.28                & 0.30                & 0.28                & 0.32                & 0.30                & 0.30                & 0.31                & 0.23                 & 0.29                & 0.28                \\
			& SHD          & 42805               & 36831               & 40464               & 35660               & 37413               & 40151               & 37297               & 50139                & 41958               & 39476               \\
			& T (s)      & 4037.6              & 2796.0              & 3331.4              & 3369.7              & 2833.4              & 3669.7              & 3755.3              & 4439.0               & 4179.0              & 2905.7              \\ \midrule
			\multirow{4}{*}{PC-Stable}        & $\text{F1}^{-}$  & -                   & -                   & -                   & -                   & -                   & -                   & -                   & -                    & -                   & -                   \\
			& $\text{F1}^{\rightarrow}$ & -                   & -                   & -                   & -                   & -                   & -                   & -                   & -                    & -                   & -                   \\
			& SHD          & -                   & -                   & -                   & -                   & -                   & -                   & -                   & -                    & -                   & -                   \\
			& T (s)      & 86400.0             & 86400.0             & 86400.0             & 86400.0             & 86400.0             & 86400.0             & 86400.0             & 86400.0              & 86400.0             & 86400.0             \\ \midrule
			\multirow{4}{*}{fGES}              & $\text{F1}^{-}$  & -                   & -                   & -                   & -                   & -                   & -                   & -                   & -                    & -                   & -                   \\
			& $\text{F1}^{\rightarrow}$ & -                   & -                   & -                   & -                   & -                   & -                   & -                   & -                    & -                   & -                   \\
			& SHD          & -                   & -                   & -                   & -                   & -                   & -                   & -                   & -                    & -                   & -                   \\
			& T (s)      & 86400.0             & 86400.0             & 86400.0             & 86400.0             & 86400.0             & 86400.0             & 86400.0             & 86400.0              & 86400.0             & 86400.0             \\ \midrule
			\multirow{3}{*}{Impro. ratio}       & $\text{F1}^{-}$  & 36.0\%              & 31.4\%              & 78.8\%              & 10.8\%              & 75.7\%              & 47.6\%              & 10.4\%              & 54.5\%               & 133.0\%             & 36.0\%              \\
			& $\text{F1}^{\rightarrow}$ & 115.6\%             & 127.2\%             & 225.4\%             & 33.9\%              & 208.0\%             & 94.7\%              & 70.9\%              & 135.0\%              & 135.8\%             & 180.0\%             \\
			& SHD          & 62.2\%              & 84.8\%              & 95.9\%              & 32.6\%              & 96.0\%              & 68.5\%              & 50.4\%              & 56.1\%               & 52.0\%              & 76.3\%             \\ \bottomrule
		\end{tabular}
	}
	\label{tab:results_10000}
\end{table*}

\subsection{Results and Analysis}
\label{sec:results}
The partition step of PEF typically results in subproblems with  5\%$\sim$10\% variables of the original problem.
For testing problems with 1000 and 10,000 variables, the subproblems have 50$\sim$100 and 500$\sim$1000 variables, respectively, which are problem sizes covered by training data.
However, as the number of variables increases further (e.g., to 30,000), the subproblems will be much larger than the training problems.

We first examine the performance on testing problems with around 1000 and 10,000 variables to answer RQ1 raised at the beginning of this section.
Specifically, based on each network selected from \textit{bnlearn}, we generate 10 testing problems (with different random seeds), test each method on these problems, and report the mean ± standard deviation in terms of the evaluation metrics as well as statistical test results in Table~\ref{tab:results_1000}.
For the performance evaluation involving 10,000 variables, due to the very long runtime of the baselines, we generate one test problem based on each network and report the testing results in Table~\ref{tab:results_10000}.

The first observation from Table~\ref{tab:results_1000} is that P/SLE consistently achieves significantly higher accuracy across all three metrics compared to the baselines, except for the F1 metrics on Healthcare testing problems, where P/SLE performs slightly  worse than PC-Stable.
Table~\ref{tab:results_10000} shows that the superiority of P/SLE becomes more pronounced on testing problems with 10,000 variables.
In call cases it achieves substantially higher accuracy than all baselines across all accuracy metrics.
Notably, compared to the best performance achieved by the baselines, P/SLE often achieves improvements in $\text{F1}^{-}$ of over 30\% and up to 133\%, and improvements in $\text{F1}^{\rightarrow}$ and SHD of over 50\% and even up to 225\%.
Since the difference between P/SLE and P/PC-Stable (P/fGES) lies in the use of a SLE in the estimation step instead of a single algorithm, the consistent advantages of P/SLE over them confirm that using SLEs can stably achieve high learning accuracy across subproblems.
On the other hand, we also observe that P/SLE(D) and P/fGES obtain identical learning accuracy.
This is because, in the default ensemble, the output of fGES always has the best BIC score among all member algorithms, thus making it consistently being chosen as the final output.
While P/SLE(R), with a randomly constructed SLE, can avoid this issue, it fails to achieve satisfactory learning accuracy, which in some cases is even worse than P/PC-Stable and P/fGES that do not use SLEs.
Therefore, the advantages of P/SLE over P/SLE(D) and P/SLE(R) highlight the effectiveness of Auto-SLE in learning high-quality SLEs with complementary member algorithms.

\begin{figure*}[tbp]
	\centering
	\includegraphics[width=\textwidth, trim=0cm 0cm 0cm 0cm, clip]{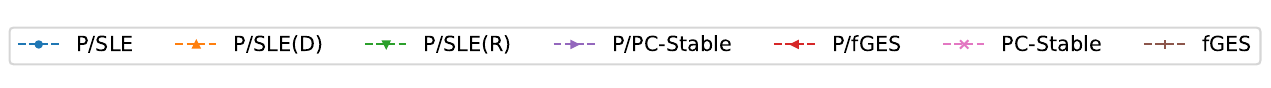}
	\vspace{-10pt}
	\subfloat[]{\includegraphics[width=.245\textwidth]{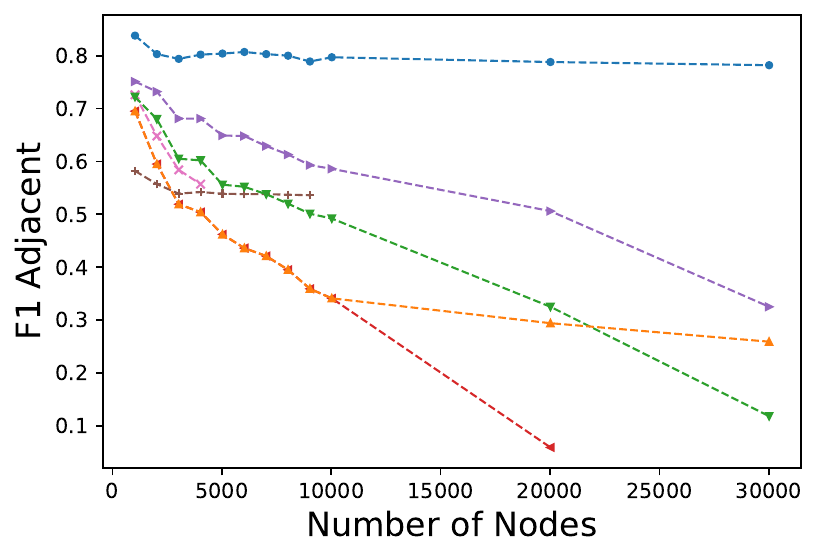}}
	\subfloat[]{\includegraphics[width=.245\textwidth]{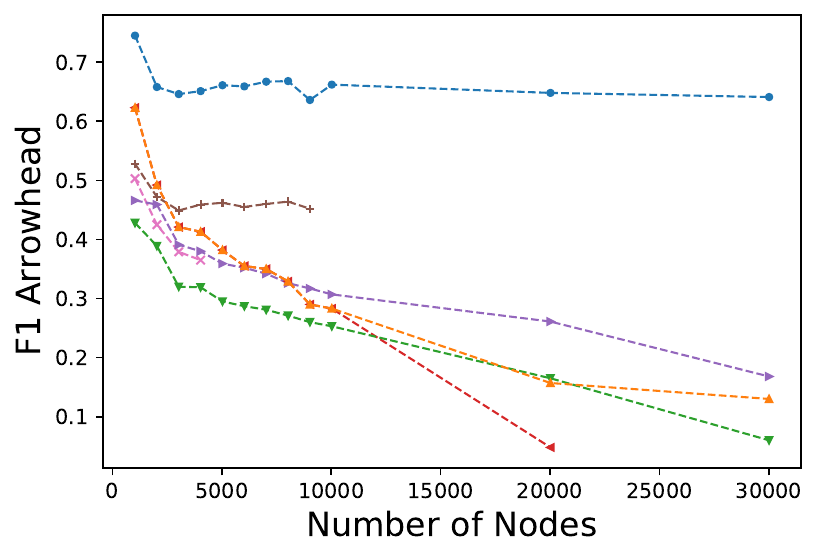}}
	\subfloat[]{\includegraphics[width=.245\textwidth]{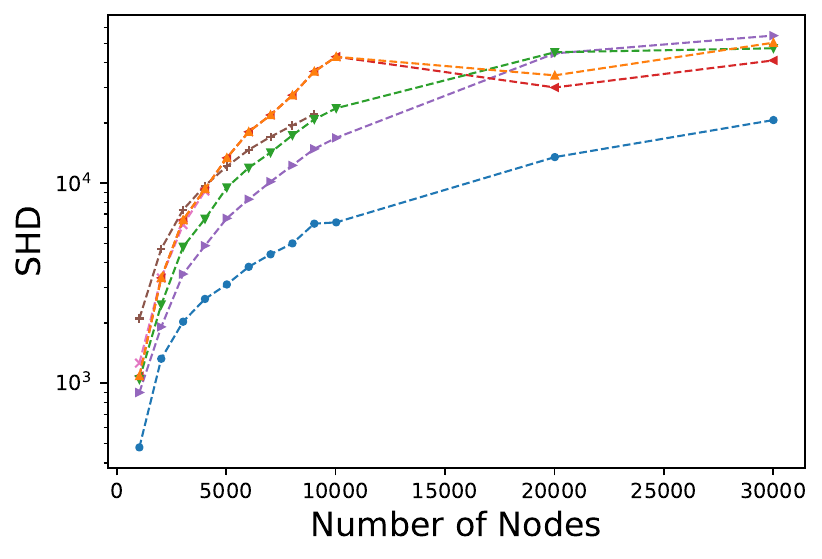}}
	\subfloat[]{\includegraphics[width=.245\textwidth]{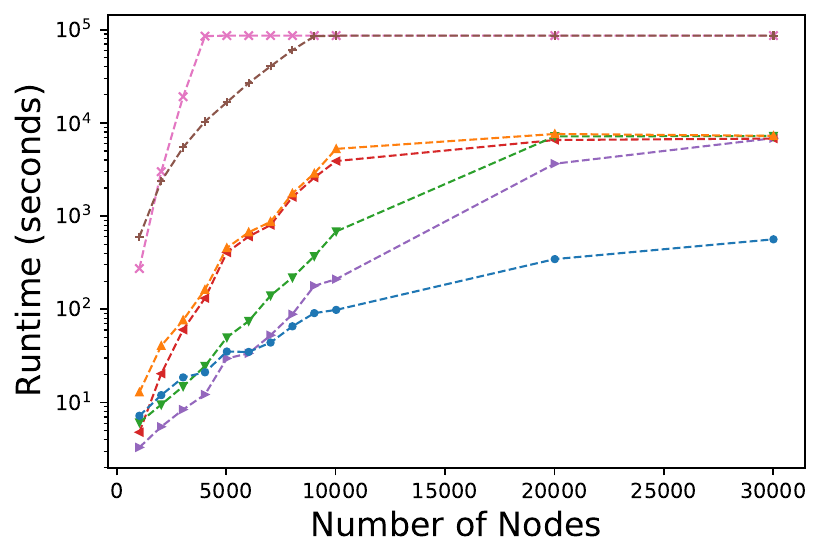}}

	\subfloat[]{\includegraphics[width=.245\textwidth]{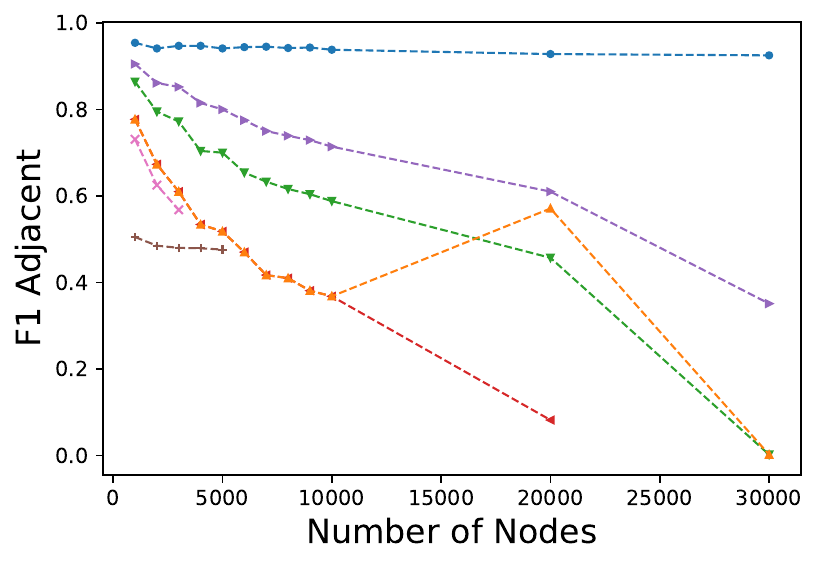}}
	\subfloat[]{\includegraphics[width=.245\textwidth]{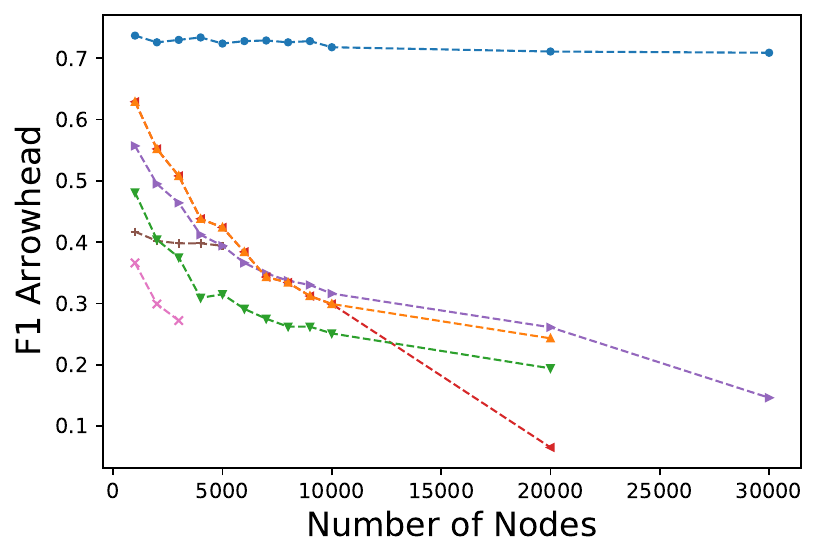}}
	\subfloat[]{\includegraphics[width=.245\textwidth]{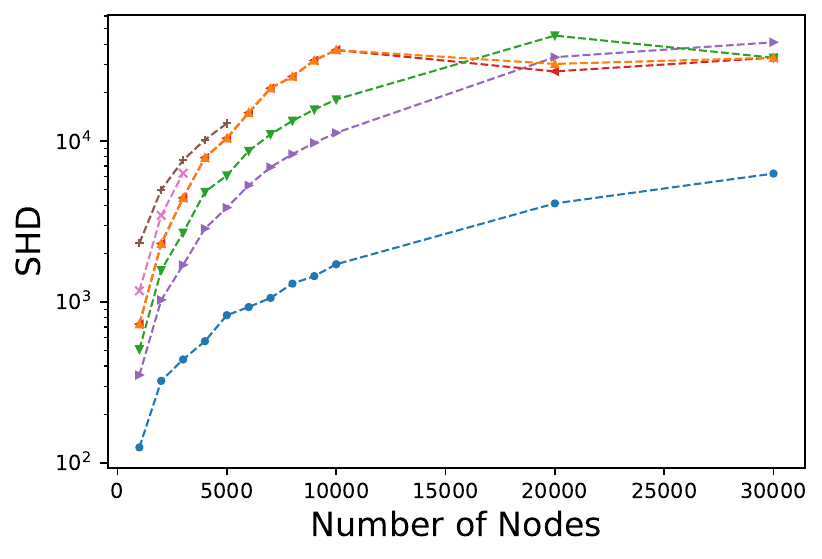}}
	\subfloat[]{\includegraphics[width=.245\textwidth]{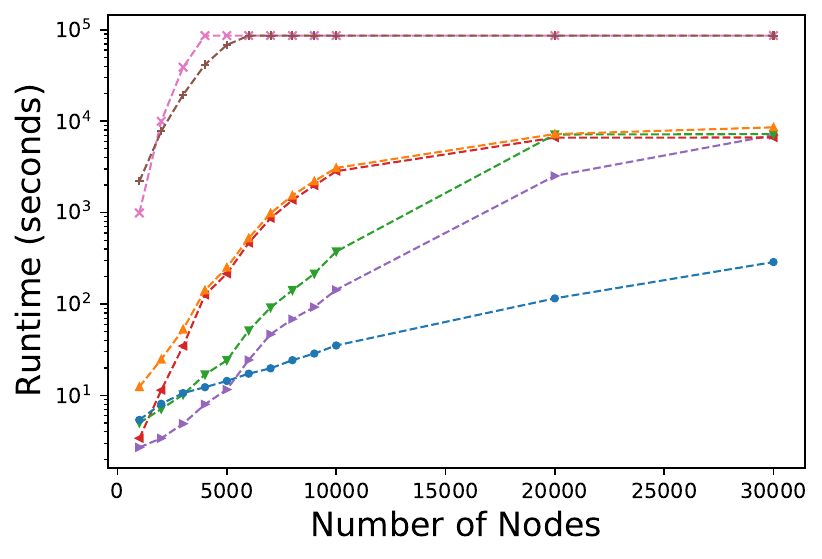}}	
	\caption{Performance curves on Alarm and Asia problems with up to 30,000 variables. SHD and runtime are plotted on log scale. a-d are for Alarm problem, e-h are for Asia problem}
	\label{fig:generalization}
\end{figure*}

\begin{table}
	\centering
	\caption{Testing results on Yeast and WS problems.
		``P/SLE-b'' represents the best performance achieved among P/SLE(R), P/SLE(D), P/fGES, and P/PC-Stable.
		The best performance in terms of accuracy metrics is marked with an \underline{underline}, and the performance that is not significantly different from the best performance is indicated in \textbf{bold}.
	}
	\resizebox{\linewidth}{!}{
		\begin{tabular}{llccccc}
			\toprule
			\multicolumn{1}{c}{Problem}        &              & WS                       & WS        & Yeast           \\
			\multicolumn{1}{c}{($|V|, |E|$)}    &         & (1000, 2000)              & (10000, 20000) & (4441, 12873)       \\ 
			\midrule
			\multirow{4}{*}{P/SLE}           & $\text{F1}^{-}$  & \underline{\textbf{0.80±0.01}} & \underline{\textbf{0.79}}  & \underline{\textbf{0.106}}\\
			& $\text{F1}^{\rightarrow}$ & \textbf{0.51±0.01}       & \underline{\textbf{0.49}}  & \underline{\textbf{0.084}}\\
			& SHD          & \underline{\textbf{1240.2±47.7}}     & \underline{\textbf{12611}} & \underline{\textbf{20912}}\\
			& T (s)      & 9.8±0.3                  & 83.2                 & 7105.4\\ 
			\midrule
			\multirow{4}{*}{P/SLE-b} & $\text{F1}^{-}$  & 0.79±0.01                & 0.72                & 0.003 \\
			& $\text{F1}^{\rightarrow}$ & \underline{\textbf{0.51±0.01}}                & 0.40               & 0.001  \\
			& SHD          & 1252.4±32.2              & 15760                & 33162\\
			& T (s)      & 3.7±0.3                  & 363.1               & 5771.32 \\ 
			\midrule
			\multirow{4}{*}{PC-Stable}        & $\text{F1}^{-}$  & 0.75±0.01                & -                  & -  \\
			& $\text{F1}^{\rightarrow}$ & 0.42±0.01                & -                   & - \\
			& SHD          & 1437.4±21.2              & -                   & - \\
			& T (s)      & 176.2±6.5                & 86400.0             & 86400.0 \\
			\midrule
			\multirow{4}{*}{fGES}              & $\text{F1}^{-}$  & 0.67±0.00                & -                  & 0.066  \\
			& $\text{F1}^{\rightarrow}$ & 0.43±0.01                & -                  & 0.056  \\
			& SHD          & 2401.0±22.3              & -                  & 29847  \\
			& T (s)      & 1433.3±61.4              & 86400.0            & 178257.4  \\
			\midrule
			\multirow{3}{*}{Impro. ratio}       & $\text{F1}^{-}$  & 2.2\%                    & 10.6\%              & 60.6\% \\
			& $\text{F1}^{\rightarrow}$ & -0.6\%                   & 21.1\%              &  50.0\%\\
			& SHD          & 1.0\%                    & 20.0\%              &  29.9\%\\
			\bottomrule
		\end{tabular}
	}
	\label{tab:no_block_results}
\end{table}

\begin{table*}[tbp]
  \centering
  \caption{Testing results on Asia problems with up to 30000 variables, in terms of F1 Adjacent, F1 Arrowhead, SHD, and runtime.
  	On each network, the mean ± std performance obtained by each method on 10 problems is reported.
  	The best performance in terms of accuracy metrics is indicated in \textbf{bold}.}
    \resizebox{1.0\textwidth}{!}{
    \begin{tabular}{cccccccccccccc}
    \toprule
    \multicolumn{2}{c}{Problem} & Asia-1000 & Asia-2000 & Asia-3000 & Asia-4000 & Asia-5000 & Asia-6000 & Asia-7000 & Asia-8000 & Asia-9000 & Asia-10000 & Asia-20000 & Asia-30000 \\
    \multicolumn{2}{c}{(|V|,|E|)} & (1000, 1100) & (2000, 2200) & (3000, 3300) & (4000, 4400) & (5000, 5500) & (6000, 6600) & (7000, 7700) & (8000, 8800) & (9000, 9900) & (10000, 11000) & (20000, 22000) & (30000, 33000) \\
    \midrule
    \multirow{4}[2]{*}{P/SLE} & F1 Adjacent & \textbf{95.4 } & \textbf{94.1 } & \textbf{94.7 } & \textbf{94.7 } & \textbf{94.1 } & \textbf{94.4 } & \textbf{94.5 } & \textbf{94.2 } & \textbf{94.3 } & \textbf{93.8 } & \textbf{92.8 } & \textbf{92.5 } \\
          & F1 Arrowhead & \textbf{73.7 } & \textbf{72.6 } & \textbf{73.0 } & \textbf{73.4 } & \textbf{72.4 } & \textbf{72.8 } & \textbf{72.9 } & \textbf{72.6 } & \textbf{72.8 } & \textbf{71.8 } & \textbf{71.1 } & \textbf{70.9 } \\
          & SHD   & \textbf{125.0 } & \textbf{324.0 } & \textbf{440.0 } & \textbf{571.0 } & \textbf{829.0 } & \textbf{931.0 } & \textbf{1061.0 } & \textbf{1302.0 } & \textbf{1449.0 } & \textbf{1717.0 } & \textbf{4103.0 } & \textbf{6287.0 } \\
          & Runtime & 5.4   & 8.1   & 10.6  & 12.3  & 14.4  & 17.3  & 19.8  & 24.3  & 28.7  & 35.2  & 115.3  & 287.8  \\
    \midrule
    \multirow{4}[2]{*}{P/SLE(D)} & F1 Adjacent & 77.7  & 67.3  & 61    & 53.4  & 51.8  & 47    & 41.7  & 41    & 38.1  & 36.8  & 57.1  & 0.2 \\
          & F1 Arrowhead & 62.9  & 55.2  & 50.8  & 43.8  & 42.4  & 38.4  & 34.3  & 33.4  & 31.2  & 29.9  & 24.3  & / \\
          & SHD   & 729   & 2307  & 4439  & 7890  & 10416 & 15021 & 21308 & 25178 & 31718 & 36831 & 30180 & 32963 \\
          & Runtime & 12.4  & 24.8  & 52.8  & 141.6 & 250.9 & 524.2 & 984.1 & 1538.6 & 2210.7 & 3083  & 7230.6 & 8571.3 \\
\cmidrule{1-1}    \multirow{4}[2]{*}{P/SLE(R)} & F1 Adjacent & 86.4  & 79.5  & 77.2  & 70.4  & 70.0  & 65.4  & 63.3  & 61.6  & 60.4  & 58.8  & 45.7  & 0.2  \\
          & F1 Arrowhead & 48.1  & 40.4  & 37.5  & 30.9  & 31.5  & 29.1  & 27.5  & 26.2  & 26.2  & 25.1  & 19.4  & / \\
          & SHD   & 508.0  & 1577.0  & 2697.0  & 4832.0  & 6110.0  & 8688.0  & 11012.0  & 13366.0  & 15712.0  & 18141.0  & 45260.0  & 32962.0  \\
          & Runtime & 5.0   & 7.1   & 10.2  & 16.9  & 24.2  & 51.4  & 91.5  & 141.2  & 214.5  & 374.8  & 7153.0  & 7220.1  \\
    \midrule
    \multirow{4}[2]{*}{P/PC-Stable} & F1 Adjacent & 90.5  & 86.1  & 85.2  & 81.5  & 80.0  & 77.5  & 75.0  & 73.9  & 72.9  & 71.4  & 61.0  & 35.1  \\
          & F1 Arrowhead & 55.7  & 49.5  & 46.4  & 41.2  & 39.4  & 36.6  & 34.9  & 33.7  & 33.0  & 31.6  & 26.1  & 14.6  \\
          & SHD   & 352.0  & 1031.0  & 1700.0  & 2863.0  & 3859.0  & 5308.0  & 6889.0  & 8301.0  & 9745.0  & 11262.0  & 33245.0  & 41131.0  \\
          & Runtime & 2.7   & 3.4   & 4.9   & 8.0   & 11.6  & 24.4  & 46.7  & 68.3  & 92.7  & 143.0  & 2521.7  & 7036.3  \\
    \midrule
    \multirow{4}[2]{*}{P/fGES} & F1 Adjacent & 77.7  & 67.3  & 61.0  & 53.4  & 51.8  & 47.0  & 41.7  & 41.0  & 38.1  & 36.8  & 8.2   & / \\
          & F1 Arrowhead & 62.9  & 55.2  & 50.8  & 43.8  & 42.4  & 38.4  & 34.3  & 33.4  & 31.2  & 29.9  & 6.5   & / \\
          & SHD   & 729.0  & 2307.0  & 4439.0  & 7890.0  & 10416.0  & 15021.0  & 21308.0  & 25178.0  & 31718.0  & 36831.0  & 27119.0  & 33000.0  \\
          & Runtime & 3.4   & 11.4  & 34.7  & 125.5  & 214.3  & 466.5  & 867.3  & 1370.8  & 1973.2  & 2838.6  & 6617.6  & 6629.0  \\
    \midrule
    \multirow{4}[2]{*}{PC-Stable} & F1 Adjacent & 73.1  & 62.5  & 56.8  & /     & /     & /     & /     & /     & /     & /     & /     & / \\
          & F1 Arrowhead & 36.6  & 29.9  & 27.2  & /     & /     & /     & /     & /     & /     & /     & /     & / \\
          & SHD   & 1174.0  & 3457.0  & 6321.0  & /     & /     & /     & /     & /     & /     & /     & /     & / \\
          & Runtime & 992.5  & 9984.5  & 38956.1  & 86400.0  & 86400.0  & 86400.0  & 86400.0  & 86400.0  & 86400.0  & 86400.0  & 86400.0  & 86400.0  \\
    \midrule
    \multirow{4}[2]{*}{fGES} & F1 Adjacent & 50.5  & 48.5  & 48.0  & 47.9  & 47.6  & /     & /     & /     & /     & /     & /     & / \\
          & F1 Arrowhead & 41.7  & 40.2  & 39.8  & 39.8  & 39.4  & /     & /     & /     & /     & /     & /     & / \\
          & SHD   & 2317.0  & 4982.0  & 7611.0  & 10208.0  & 12878.0  & /     & /     & /     & /     & /     & /     & / \\
          & Runtime & 2202.0  & 7762.0  & 19452.0  & 41567.9  & 68229.9  & 86400.0  & 86400.0  & 86400.0  & 86400.0  & 86400.0  & 86400.0  & 86400.0  \\
    \bottomrule
    \end{tabular}}
  \label{tab:big_asia}
\end{table*}

\begin{table*}[tbp]
  \centering
      		\caption{Testing results on Alarm problems with up to 30000 variables, in terms of F1 Adjacent, F1 Arrowhead, SHD, and runtime.
  	On each network, the mean ± std performance obtained by each method on 10 problems is reported.
  	The best performance in terms of accuracy metrics is indicated in \textbf{bold}.}
  	\resizebox{1.0\textwidth}{!}{
    \begin{tabular}{llcccccccccccc}
    \toprule
    \multicolumn{2}{c}{Problem} & Alarm-1000 & Alarm-2000 & Alarm-3000 & Alarm-4000 & Alarm-5000 & Alarm-6000 & Alarm-7000 & Alarm-8000 & Alarm-9000 & Alarm-10000 & Alarm-20000 & Alarm-30000 \\
    \multicolumn{2}{c}{($|V|,|E|$)} & (1036, 1417) & (2035, 2783) & (3034, 4150) & (4033, 5516) & (5032, 6882) & (6031, 8248) & (7030, 9614) & (8029, 10981) & (9028, 12347) & (10027, 13713) & (20017, 27375) & (30007, 41037) \\
    \midrule
    \multirow{4}[2]{*}{P/SLE} & F1 Adjacent & \textbf{83.8 } & \textbf{80.3 } & \textbf{79.4 } & \textbf{80.2 } & \textbf{80.4 } & \textbf{80.7 } & \textbf{80.3 } & \textbf{80.0 } & \textbf{78.9 } & \textbf{79.7 } & \textbf{78.8 } & \textbf{78.2 } \\
          & F1 Arrowhead & \textbf{74.5 } & \textbf{65.8 } & \textbf{64.6 } & \textbf{65.1 } & \textbf{66.1 } & \textbf{65.9 } & \textbf{66.7 } & \textbf{66.8 } & \textbf{63.6 } & \textbf{66.2 } & \textbf{64.8 } & \textbf{64.1 } \\
          & SHD   & \textbf{478.0 } & \textbf{1326.0 } & \textbf{2030.0 } & \textbf{2638.0 } & \textbf{3117.0 } & \textbf{3817.0 } & \textbf{4409.0 } & \textbf{4999.0 } & \textbf{6272.0 } & \textbf{6366.0 } & \textbf{13491.0 } & \textbf{20658.0 } \\
          & Runtime (s) & 7.2   & 12.0  & 18.6  & 21.1  & 35.2  & 34.9  & 44.0  & 65.7  & 91.0  & 98.6  & 345.6  & 563.6  \\
    \midrule
    \multirow{4}[2]{*}{P/SLE(D)} & F1 Adjacent & 69.5  & 59.5  & 51.9  & 50.4  & 46.2  & 43.6  & 42.1  & 39.5  & 35.9  & 34.1  & 29.4  & 25.9 \\
          & F1 Arrowhead & 62.3  & 49.2  & 42.1  & 41.3  & 38.2  & 35.5  & 35    & 32.9  & 29    & 28.3  & 15.7  & 13 \\
          & SHD   & 1086  & 3361  & 6549  & 9357  & 13371 & 18043 & 21892 & 27495 & 36108 & 42805 & 34557 & 50319 \\
          & Runtime (s) & 12.9  & 40.6  & 76.5  & 162.3 & 455.9 & 670.4 & 868.5 & 1755  & 2874.8 & 5262.8 & 7603.4 & 7263 \\
    \midrule
    \multirow{4}[2]{*}{P/SLE(R)} & F1 Adjacent & 72.2  & 68.0  & 60.5  & 60.2  & 55.6  & 55.2  & 53.8  & 52.0  & 50.1  & 49.2  & 32.5  & 11.8  \\
          & F1 Arrowhead & 42.8  & 38.9  & 32.0  & 31.9  & 29.5  & 28.7  & 28.1  & 27.1  & 26.0  & 25.3  & 16.5  & 6.0  \\
          & SHD   & 1046.0  & 2470.0  & 4802.0  & 6626.0  & 9501.0  & 11931.0  & 14221.0  & 17349.0  & 20817.0  & 23673.0  & 45085.0  & 47184.0  \\
          & Runtime (s) & 6.1   & 9.5   & 14.9  & 24.6  & 49.9  & 74.8  & 140.2  & 219.8  & 371.6  & 684.6  & 7166.5  & 7229.8  \\
    \midrule
    \multirow{4}[2]{*}{P/PC-Stable} & F1 Adjacent & 75.1  & 73.2  & 68.1  & 68.1  & 64.9  & 64.8  & 62.9  & 61.3  & 59.3  & 58.6  & 50.6  & 32.5  \\
          & F1 Arrowhead & 46.6  & 45.9  & 39.1  & 38.0  & 35.9  & 35.2  & 34.2  & 32.6  & 31.7  & 30.7  & 26.1  & 16.8  \\
          & SHD   & 899.0  & 1913.0  & 3512.0  & 4869.0  & 6670.0  & 8306.0  & 10163.0  & 12269.0  & 14830.0  & 16842.0  & 44491.0  & 54599.0  \\
          & Runtime (s) & 3.3   & 5.5   & 8.4   & 12.2  & 29.7  & 33.7  & 52.7  & 88.6  & 178.5  & 210.8  & 3649.3  & 6853.3  \\
    \midrule
    \multirow{4}[2]{*}{P/fGES} & F1 Adjacent & 69.5  & 59.5  & 51.9  & 50.4  & 46.2  & 43.6  & 42.1  & 39.5  & 35.9  & 34.1  & 5.9   & / \\
          & F1 Arrowhead & 62.3  & 49.2  & 42.1  & 41.3  & 38.2  & 35.5  & 35.0  & 32.9  & 29.0  & 28.3  & 4.8   & / \\
          & SHD   & 1086.0  & 3361.0  & 6549.0  & 9357.0  & 13371.0  & 18043.0  & 21892.0  & 27495.0  & 36108.0  & 42805.0  & 30065.0  & 41037.0  \\
          & Runtime (s) & 4.8   & 20.4  & 60.3  & 131.7  & 406.7  & 604.0  & 801.4  & 1600.0  & 2585.5  & 3896.7  & 6549.9  & 6792.9  \\
    \midrule
    \multirow{4}[2]{*}{PC-Stable} & F1 Adjacent & 72.6  & 64.8  & 58.4  & 55.7  & /     & /     & /     & /     & /     & /     & /     & / \\
          & F1 Arrowhead & 50.3  & 42.5  & 37.9  & 36.5  & /     & /     & /     & /     & /     & /     & /     & / \\
          & SHD   & 1261.0  & 3360.0  & 6192.0  & 9115.0  & /     & /     & /     & /     & /     & /     & /     & / \\
          & Runtime (s) & 273.5  & 2997.3  & 19091.2  & 85197.4  & 86400.0  & 86400.0  & 86400.0  & 86400.0  & 86400.0  & 86400.0  & 86400.0  & 86400.0  \\
    \midrule
    \multirow{4}[2]{*}{fGES} & F1 Adjacent & 58.2  & 55.7  & 53.9  & 54.2  & 53.9  & 53.8  & 53.8  & 53.7  & 53.6  & /     & /     & / \\
          & F1 Arrowhead & 52.8  & 47.2  & 44.9  & 45.9  & 46.2  & 45.5  & 46.0  & 46.4  & 45.2  & /     & /     & / \\
          & SHD   & 2107.0  & 4697.0  & 7364.0  & 9707.0  & 12123.0  & 14678.0  & 17045.0  & 19423.0  & 22128.0  & /     & /     & / \\
          & Runtime (s) & 600.9  & 2369.1  & 5453.8  & 10336.8  & 16623.9  & 26760.6  & 40465.7  & 60305.9  & 84889.9  & 86400.0  & 86400.0  & 86400.0  \\
    \bottomrule
    \end{tabular}}
  \label{tab:big_alarm}
\end{table*}

The second observation is that on testing problems with 1000 variables, P/PC-Stable and P/fGES often achieve higher learning accuracy than PC-Stable and fGES, respectively.
Moreover, when the number of variables reaches 10,000, PC-Stable and fGES are unable to find solutions within 24 hours, while P/PC-Stable and P/fGES are able to.
These findings show that D\&C strategies such as PEF can indeed enhance the capabilities of handling large BNs, which is consistent with the observations in~\cite{gu2020learning}.

Finally, all PEF-based methods generally consume much less runtime than non-PEF-based methods, attributed to the underlying D\&C strategy.
Among PEF-based methods, P/SLE often has the shortest or close to the shortest runtime, and for most testing problems, it consistently outputs the final solution within a reasonable time (less than 1000 seconds).
In summary, all the above findings affirmatively answer RQ1, i.e., the SLE learned by Auto-SLE substantially improves PEF in learning the structure of large BNs.

\subsection{Generalization to Larger Problems}
We now investigate RQ2.
Specifically, we generate Alarm and Asia testing problems with 20,000 and 30,000 variables, and plot in Figure~\ref{fig:generalization} the performance of the compared methods as the variable number ranges from 1000 to 30,000 (detailed results are presented in Table~\ref{tab:big_asia} and ~\ref{tab:big_alarm}).
Note that when the variable number exceeds 20,000, the subproblems resulted from the partition step of PEF would be much larger in size than the training problems.
It can be seen from Figure~\ref{fig:generalization} that P/SLE generalizes well to larger problems, maintaining relatively stable learning accuracy. 
In contrast, the performance of all baselines deteriorates rapidly as the number of variables increases.

\subsection{Generalization to Problems with no Block Structure}
The above results have demonstrated that P/SLE can stably achieve high learning accuracy for networks with a block structure.
Of course, there are many real-world networks without any block structure~\cite{OlesenM02}.
Although PEF-based methods are not specifically designed for such networks, it is worth testing P/SLE on them to provide a complete spectrum of its performance.
Specifically, we apply P/SLE for gene expression data analysis, a traditional application of structure leanring for BN.
We use the real-world gene dataset Yeast~\cite{schaffter2011genenetweaver} involving 4,441 nodes and 12,873 edges, where the underlying networks are commonly referred to as gene regulatory networks.
Besides, we use small-world networks with 1000 and 10,000 nodes, based on \textit{igraph}~\cite{csardi2006igraph} and the the Watts-Strogatz (WS) model~\cite{watts1998collective}.
We choose the WS model because: (i) it has no block structure, and (ii) it is not included in the \textit{bnlearn} repository, thereby enabling evaluation of the generalization of P/SLE to network characteristics beyond the training set.
As before, we generate 10 testing problems based on the network with 1000 nodes and one testing problem based on the network with 10,000 nodes.
The testing results are presented in Table~\ref{tab:no_block_results}.

One can observe that P/SLE still achieves competitive learning accuracy on these networks. 
On WS problems with 1000 variables, it always achieves the best performance or the performance not significantly different from the best, across all accuracy metrics.
On WS problems with 10,000 variables and Yeast, the advantages of P/SLE become pronounced, similar to the previous observations on networks with block structures.
In summary, these findings show the generalization ability of the learned SLE across network characteristics beyond the training problem set.

\section{Discussion and Conclusion}
\label{sec:conclusion}
This work introduced the idea of employing SLEs for scalable BN structure learning.
To facilitate the development of such ensembles, the SLE learning problem was formally modeled, leading to the proposal of Auto-SLE, a simple yet effective greedy approach for automatically learning high-quality SLEs.
The theoretical properties of Auto-SLE were also analyzed, demonstrating its ability to yield near-optimal SLEs for a given training set.
The SLE learned by Auto-SLE was integrated into the estimation step of PEF, resulting in P/SLE.
Extensive experiments conducted on both synthetic and real-world datasets demonstrated that P/SLE consistently achieved high accuracy in learning large BNs and generalized well across problem sizes and network characteristics.
These results have indicated the potential of applying SLEs to further improve D\&C methods for learning large BNs and have verified the effectiveness of Auto-SLE for learning SLEs.

Several directions exist for future work.
First, the learned SLE can be integrated into other D\&C methods beyond PEF, such as~\cite{ZhangZYGWZH22,WangBCLZC25}.
Second, while using an SLE by running its member algorithms in parallel would not significantly increase the wall-clock runtime compared to running a single algorithm, executing them sequentially in the absence of multi-core compute leads to significantly longer runtime.
A potential solution involves training a selection model~\cite{ZhaoLYX2021} that predicts the best-performing algorithm within the SLE for a given problem, allowing only that specific algorithm to be run.
Third, the generalization ability of the SLE learned by Auto-SLe relies on a training set with diverse networks of varying sizes.
Future research could explore alternative strategies for acquiring suitable training data when such a diverse set cannot be readily collected, such as the use of large language models (LLMs) to assist in data creation.


\bibliographystyle{IEEEtran}
\bibliography{Auto-SLE}


\end{document}